\documentclass{tlp}
\usepackage{url}
\usepackage{graphicx}
\usepackage{subfigure}
\usepackage[dvips,all]{xy}
\usepackage[usenames]{color}  

\newcommand\cA{{\cal A}}

\newcommand\cE{{\cal E}}
\newcommand\cF{{\cal F}}
\newcommand\cH{{\cal H}}
\newcommand\cS{{\cal S}}
\newcommand\cT{{\cal T}}
\newcommand\cW{{\cal W}}

\newcommand\boundedtermsize{bounded term-size}
\newcommand\Boundedtermsize{Bounded Term-size}

\newtheorem{theorem}{Theorem} 
\newtheorem{lemma}{Lemma} 
\newtheorem{definition}{Definition} 
\newtheorem{example}{Example} 

\newcommand{\algorithm}{PITA}

 \submitted {22 November 2010}
 \revised {15 April 2011}
 \accepted{14 June 2011}

\begin{document}

\long\def\comment#1{}

\title[Well\--Definedness and Efficient Inference for Probabilistic Logic
Prog.]{Well\--Definedness and Efficient Inference for Probabilistic Logic
Programming under the Distribution Semantics}
\author[F. Riguzzi and T. Swift]
{FABRIZIO RIGUZZI \\
ENDIF -- University of Ferrara\\
Via Saragat 1, I-44122, Ferrara, Italy  \\
E-mail: fabrizio.riguzzi@unife.it
\and TERRANCE SWIFT \\
CENTRIA -- Universidade Nova de Lisboa \\
E-mail: tswift@cs.suysb.edu
}

\pagerange{\pageref{firstpage}--\pageref{lastpage}}
\setcounter{page}{1}
\maketitle
\label{firstpage}

\noindent
{\bf Note}: This article will appear in Theory and Practice of Logic Programming, 
\copyright Cambridge University Press, \\

\begin{abstract}
The distribution semantics is one of the most prominent approaches for the combination of logic programming and probability theory. Many languages follow this semantics, such as Independent Choice Logic, PRISM, pD, Logic Programs with Annotated Disjunctions (LPADs)  and ProbLog. 

When a program contains functions symbols, the distribution semantics
is well\--defined only if the set of explanations for a query is
finite and so is each explanation. Well\--definedness is usually
either explicitly imposed or is achieved by severely limiting the
class of allowed programs.
In this paper we identify a larger class of programs for which the
semantics is well\--defined together with an efficient procedure for
computing the probability of queries.
Since LPADs offer the most general syntax, we present our results for
them, but our results are applicable to all languages under the
distribution semantics.

We present the algorithm ``Probabilistic Inference with Tabling and
Answer subsumption'' (\algorithm) that computes the probability of
queries by transforming a probabilistic program into a normal program
and then applying SLG resolution with answer subsumption.
\algorithm\ has been implemented in XSB and tested on six domains: two
with function symbols and four without.  The execution times are
compared with those of ProbLog, \texttt{cplint} and
CVE. \algorithm\ was almost always able to solve larger problems in a
shorter time, on domains with and without function symbols.
\end{abstract}


\section{Introduction}
\label{intro}
Many real world domains can only be represented effectively if we are
able to model uncertainty.  Accordingly, there has been an increased
interest in logic languages representing probabilistic information,
stemming in part from their successful use in Machine Learning.  In
particular, languages that follow the distribution semantics
\cite{DBLP:conf/iclp/Sato95} have received much attention in the last
few years. In these languages a theory defines a probability
distribution over logic programs, which is extended to a joint
distribution over programs and queries.  The probability of a query is
then obtained by marginalizing out the programs.

Examples of languages that follow the distribution semantics are
Independent Choice Logic \cite{DBLP:journals/ai/Poole97}, PRISM \cite{DBLP:conf/ijcai/SatoK97}, pD \cite{DBLP:journals/jasis/Fuhr00},  Logic Programs with Annotated Disjunctions (LPADs)  \cite{VenVer04-ICLP04-IC} and ProbLog \cite{DBLP:conf/ijcai/RaedtKT07}. All these languages have the same expressive power as a theory in one language can be translated into another \cite{VenVer03-TR,DeR-NIPS08}.
LPADs offer the most general syntax as the constructs of all the other languages can be directly encoded in LPADs.

When programs contain functions symbols, the distribution semantics
has to be defined in a slightly different way: as proposed in
\cite{DBLP:conf/iclp/Sato95} and \cite{DBLP:journals/ai/Poole97}: the
probability of a query is defined with reference to a covering set of
explanations for the query. For the semantics to be well\--defined,
both the covering set and each explanation it contains must be finite.
To ensure that the semantics is well\--defined,
\cite{DBLP:journals/ai/Poole97} requires programs to be acyclic, while
\cite{DBLP:conf/ijcai/SatoK97} directly imposes the condition that
queries must have a finite covering set of finite explanations.

Since acyclicity is a strong requirement ruling out many interesting
programs, in this paper we propose a looser requirement to ensure the
well\--definedness of the semantics.
We introduce a definition of {\em bounded term\--size} programs and
queries, which are based on a characterization of the Well\--Founded
Semantics in terms of an iterated fixpoint \cite{Przy89d}.  A bounded
term\--size program is such that in each iteration of the fixpoint the
size of true atoms does not grow indefinitely. A bounded term\--size
query is such that the portion of the program relevant to the query is
bounded term\--size. We show that if a query is bounded term\--size,
then it has a finite set of finite explanations that are covering, so
the semantics is well\--defined.

We also present the algorithm ``Probabilistic Inference with Tabling
and Answer subsumption'' (\algorithm) that builds explanations for
every subgoal encountered during a derivation of a query.  The
explanations are compactly represented using Binary Decision Diagrams
(BDDs) that also allow an efficient computation of the probability.
Specifically, \algorithm\ transforms the input LPAD into a normal
logic program in which the subgoals have an extra argument storing a
BDD that represents the explanations for its answers.  As its name
implies, \algorithm\ uses tabling to store explanations for a
goal. Tabling has already been shown useful for probabilistic logic
programming in
\cite{DBLP:conf/cl/KameyaS00,Rig08-ICLP08-IC,KimGutSan-ILP09-IC,DBLP:conf/iclp/MantadelisJ10,RigSwi11-ICLP11-IJ}.
However, \algorithm\ is novel in its exploitation of a tabling feature
called answer subsumption to combine explanations coming from
different clauses.

\algorithm\ draws inspiration from \cite{DBLP:conf/ijcai/RaedtKT07},
which first proposed to use BDDs for computing the probability of
queries for the ProbLog language, a minimalistic probabilistic
extension of Prolog; and from \cite{Rig-AIIA07-IC} which applied BDDs
to the more general LPAD syntax. Other approaches for reasoning on
LPADs include \cite{Rig08-ICLP08-IC}, where SLG resolution is extended
by repeatedly branching on disjunctive clauses, and the CVE system
\cite{MeeStrBlo08-ILP09-IC} which transforms LPADs into an equivalent
Bayesian network and then performs inference on the network using the
variable elimination algorithm.

\algorithm\  was tested on a number of datasets, both with and without function symbols, in order to evaluate its efficiency.
The execution times of \algorithm\ were compared with those of \texttt{cplint} \cite{Rig-AIIA07-IC}, CVE \cite{MeeStrBlo08-ILP09-IC} and ProbLog \cite{ProbLog-impl}. \algorithm\ was able to solve successfully more complex queries than the other algorithms in most cases and it was also almost always faster both on datasets with and without function symbols.

The paper is organized as follows.  Section \ref{LPAD} illustrates the
syntax and semantics of LPADs over finite universes. Section
\ref{sem_fs} discusses the semantics of LPADs with function
symbols. Section \ref{mod-strat} defines dynamic stratification for
LPADs, provides conditions for the well-definedness of the LPAD
semantics with function symbols, and discusses related work on
termination of normal programs.  Section \ref{bdd} gives an
introduction to BDDs. Section \ref{tabling} briefly recalls tabling
and answer subsumption.  Section \ref{algorithm} presents \algorithm{}
and Section \ref{correctness} shows its correctness. Section \ref{related} discusses related work. Section \ref{exp}
describes the experiments and Section \ref{conc} discusses the results
and presents directions for future works.



\section{The Distribution Semantics for Function\--free Programs} \label{sec:function-free}
\label{LPAD}
In this section we illustrate the distribution semantics for
function\--free program using LPADs as the prototype of the languages
following this semantics.

A \emph{Logic Program with Annotated Disjunctions} \cite{VenVer04-ICLP04-IC} consists of a finite set of
annotated disjunctive clauses 
of the form
\[H_{1}:\alpha_{1}\vee \ldots\vee H_{n}:\alpha_{n}\leftarrow  L_{1},\ldots,L_{m}.\]
In
such a clause $H_{1},\ldots H_{n}$ are  logical atoms, $B_{1},\ldots ,B_{m}$ logical literals, and   $\alpha_{1},$ $\ldots,\alpha_{n}$ real numbers in the interval $[0,1]$
such that $\sum_{j=1}^n \alpha_{j}\leq 1$.  The term $H_{1}:\alpha_{1}\vee \ldots\vee H_{n}:\alpha_{n}$ is called the \emph{head} and $L_{1},\ldots,L_{m}$ is called the \emph{body}.  Note that if $n=1$ and $\alpha_1 = 1$ a clause corresponds to a normal program clause, also called a {\em non-disjunctive} clause.
If $\sum_{j=1}^n \alpha_{j}<1$, the head of the  clause implicitly contains an extra atom $null$ that does not appear in the body of any clause and whose annotation is $1-\sum_{j=1}^n \alpha_{j}$.
For a clause $C$, 
we  define $head(C)$ as $\{(H_i:\alpha_i)|1\leq i \leq
n\}$ if $\sum_{i=1}^n \alpha_i= 1$; and as  $\{(H_i:\alpha_i)|1\leq i \leq
n\}\cup\{(null:1-\sum_{i=1}^n \alpha_i)\}$ otherwise. Moreover, we define $body(C)$ as  $\{L_i|1\leq i \leq m\}$, $H_i(C)$ as $H_i$ and $\alpha_i(C)$ as $\alpha_i$.


If the LPAD is ground, a clause  represents a
probabilistic choice between the  non-disjunctive
clauses obtained by selecting only one atom in the head. 
%
As usual, if the LPAD $T$ is not ground, $T$ is assigned a meaning by computing its grounding, $ground(T)$.

By choosing a head  atom for each ground clause of an
LPAD we  get a normal logic program called a \emph{world} of
the LPAD (an \emph{instance} of the LPAD in \cite{VenVer04-ICLP04-IC}). 
A probability distribution is defined over the space  of worlds by assuming
independence between the choices made for each clause. 

More specifically, an  \emph{atomic choice} is a triple $(C,\theta,i)$ where $C\in T$, $\theta$ is a minimal substitution that grounds $C$ and $i\in\{1,\ldots,|head(C)|\}$. $(C,\theta,i)$  means that, for the ground clause $C\theta$, the head $H_i(C)$ was chosen.
 A set of atomic choices $\kappa$ is \emph{consistent} if $(C,\theta,i)\in\kappa,(C,\theta,j)\in \kappa\Rightarrow i=j$, i.e., only one head is selected for a ground clause. A  \emph{composite choice} $\kappa$ is a consistent set of atomic choices. The \emph{probability $P(\kappa)$ of a composite choice} $\kappa$ is the product of the probabilities of
the individual atomic choices, i.e.
$P(\kappa)=\prod_{(C,\theta,i)\in \kappa}\alpha_i(C)$.

A \emph{selection} $\sigma$ is a  composite choice that, for each  clause $C\theta$ in $ground(T)$, contains an atomic choice 
$(C,\theta,i)$ in $\sigma$. Since $T$ does not contain function symbols, $ground(T)$ is finite and so is each $\sigma$. We denote the set of all
selections  $\sigma$ of a program $T$ by ${\mathcal{S}}_T$.
%
A selection $\sigma$ identifies a normal logic program $w_\sigma$,
called a \emph{world} of $T$, defined as:
$w_\sigma=\{(H_i(C)\theta\leftarrow body(C))\theta| (C,\theta,i)\in
\sigma\}$.  ${\mathcal{W}}_T$ denotes the set of all the 
worlds of $T$.  Since selections are composite choices, we can assign
a probability to  worlds:
$P(w_\sigma)=P(\sigma)=\prod_{(C,\theta,i)\in \sigma}\alpha_i(C)$.

Throughout this paper, we consider only {\em sound} LPADs, in which
every world has a total model according to the Well-Founded Semantics
(WFS)~\cite{well-founded}. In this way, uncertainty is modeled only by means
of the disjunctions in the head and not by the semantics of negation.
Thus in the following, $w_\sigma\models A$ means that the ground atom
$A$ is true in the well-founded model of the program
$w_\sigma$\footnote{We sometimes abuse notation slightly by saying
  that an atom $A$ is true in a world $w$ to indicate that $A$ is true
  in the (unique) well-founded model of $w$.}.

%

%

In order to define the probability of an atom $A$ being true in an LPAD $T$, note that the probability distribution over possible worlds induces a probability distribution over Herbrand interpretations by  assuming $P(I|w)=1$ if $I$ is the well\--founded model  of $w$  ($I=WFM(w)$) and 0 otherwise. We can thus compute the probability of an interpretation $I$ as
$$P(I)=\sum_{w \in {\mathcal{W}}_T}P(I,w)=\sum_{w \in {\mathcal{W}}_T}P(I|w)P(w)=\sum_{w\in {\mathcal{W}}_T,I=WFM(w)}P(w).$$
We can extend the probability distribution on interpretation to ground atoms by assuming  $P(a_j|I)=1$ if $A_j$ belongs to $I$ and 0 otherwise, where $A_j$ is a ground atom of the Herbrand base $\mathcal{H}_T$ and $a_j$  stands for $A_j=true$. 
Thus the probability of  a ground atom $A_j$ being true, according to an LPAD $T$ can be obtained as
$$P(a_j)=\sum_{I}P(a_j,I)=\sum_{I}P(a_j|I)P(I)=\sum_{I \subseteq \mathcal{H}_T, A_j\in I}P(I).$$
Alternatively, we can extend the probability distribution on programs to ground atoms by assuming  $P(a_j|w)=1$ if $A_j$ is true in $w$ and 0 otherwise. Thus the probability of   $A_j$ being true is 
$$P(a_j)=\sum_{w\in {\mathcal{W}}_T}P(a_j,w)=\sum_{w\in {\mathcal{W}}_T}P(a_j|w)P(w)= \sum_{w\in {\mathcal{W}}_T,w\models A_j}P(w).$$  
The probability of $A_j$ being false is defined similarly.


\begin{example} \label{sneezing_lpad}
 Consider the dependency of sneezing on having the flu or hay fever:

{\small
$
 \begin{array}{llll}
 C_1=&strong\_sneezing(X):0.3\vee moderate\_sneezing(X):0.5&\leftarrow& flu(X).\\
 C_2=&strong\_sneezing(X):0.2\vee moderate\_sneezing(X):0.6&\leftarrow& hay\_fever(X).\\
C_3=&flu(david).\\
 C_4=&hay\_fever(david).\\
 \end{array}$}

\noindent This program models the fact that sneezing can be caused by flu or hay fever.
%
The query $moderate\_sneezing(david)$ is true in 5 of the 9 worlds of the program and its probability of being true is

%
\noindent 
$P_T(moderate\_sneezing(david))=0.5\cdot 0.2+0.5\cdot 0.6+ 0.5\cdot 0.2+0.3\cdot 0.6+ 0.2\cdot 0.6=0.8$
\end{example}
Even if we assumed independence between the choices for individual ground clauses, this does not represents a restriction, in the sense that this still allows to represent all the joint distributions of atoms of the Herbrand base that are representable with a Bayesian network over those variables. Details of the proof are omitted for lack of space.


\section{The Distribution Semantics for Programs with Function Symbols}
\label{sem_fs}
If a non-ground LPAD $T$ contains function symbols, then the semantics given in the previous section is not well-defined.  In this case,  each world $w_\sigma$ is the result of an infinite number of choices and the probability $P(w_\sigma$) is 0 since it is given by the product of an infinite number of factors all smaller than 1. Thus, the probability of a formula is 0 as well, since it is a sum of terms all equal to 0.
The distribution semantics with function symbols was defined in  \cite{DBLP:conf/iclp/Sato95} and \cite{DBLP:journals/jlp/Poole00}.  Here we follow the approach of \cite{DBLP:journals/jlp/Poole00}.

A composite choice $\kappa$ identifies a set of worlds $\omega_\kappa$ that contains all the worlds associated to a selection that is a superset of $\kappa$: i.e.,
$\omega_\kappa=\{w_\sigma|\sigma \in \mathcal{S}_T, \sigma \supseteq \kappa\}$
We define the set of worlds identified by a set of composite choices $K$ as
$\omega_K=\bigcup_{\kappa\in K}\omega_\kappa$

Given a ground atom $A$, we define the notion of explanation, covering set of composite choices and mutually incompatible set of explanations. 
A composite choice $\kappa$ is an \emph{explanation} for $A$ if $A$ is true in every world of $\omega_\kappa$. In Example \ref{sneezing_lpad}, the composite choice
$\{(C_1,\{X/david\},1)\}$
is an explanation for $strong\_sneezing(david)$.
A set of composite choices $K$ is \emph{covering} with respect to $A$ if every world $w_\sigma$ in which $A$ is true is such that $w_\sigma\in\omega_K$.
In Example \ref{sneezing_lpad}, the set of  composite choices
\begin{equation}
\label{expls}
K_1=\{\{(C_1,\{X/david\},2)\},\{(C_2,\{X/david\},2)\}\}
\end{equation}
 is covering for $moderate\_sneezing(david)$.
Two composite choices $\kappa_1$ and $\kappa_2$ are \emph{incompatible} if their union is inconsistent, i.e., if there exists a clause $C$ and a substitution $\theta$ grounding $C$ such that $(C,\theta,j)\in \kappa_1,(C,\theta,k)\in \kappa_2$ and $j\neq k$. A set $K$ of composite choices is \emph{mutually incompatible} if for all $\kappa_1\in K,\kappa_2\in K, \kappa_1\neq\kappa_2\Rightarrow \kappa_1$ and $\kappa_2$ are incompatible. As illustration, the set of composite choices
\begin{eqnarray}
K_2&=&\{\{(C_1,\{X/david\},2),(C_2,\{X/david\},1)\},\nonumber\\
&&\{(C_1,\{X/david\},2),(C_2,\{X/david\},3)\},\label{Ls}\\
&&\{(C_2,\{X/david\},2)\}\}\nonumber
\end{eqnarray}
is mutually incompatible for the theory of Example \ref{sneezing_lpad}. 
\cite{DBLP:journals/jlp/Poole00} proved the following results 
\begin{itemize}
\item Given a finite set $K$ of finite composite choices, there exists a finite set $K'$ of mutually incompatible finite composite choices such that $\omega_K=\omega_{K'}$. 
\item If $K_1$ and $K_2$ are both mutually incompatible finite sets of finite composite
choices such that $\omega_{K_1}=\omega_{K_2}$ then
$\sum_{\kappa\in K_1}P(\kappa)=\sum_{\kappa\in K_2}P(\kappa)$
\end{itemize}
Thus, we can define a unique probability measure $\mu: \Omega_T\rightarrow [0,1]$
where
$\Omega_T$ is defined as the set of sets of worlds identified by finite sets of finite composite choices: $\Omega_T=\{\omega_K|K \mbox{ is a finite set of finite composite choices}\}$.
It is easy to see that $\Omega_T$ is an algebra over $\mathcal{W}_T$.
Then $\mu$ is defined by
$\mu(\omega_K)=\sum_{\kappa\in K'}P(\kappa)$
where $K'$ is a finite mutually incompatible set of finite composite choices such that $\omega_K=\omega_{K'}$.
As is the case for ICL, $\langle\mathcal{W}_T,\Omega_T,\mu\rangle$ is a probability space 
 \cite{Kolmogorov}.

\begin{definition}\label{prob_def_LPAD}
The probability of a ground atom $A$ is given by 
$P(A)=\mu(\{w|w\in \mathcal{W}_T\wedge w\models A\}$
\end{definition}
If $A$  has a finite set $K$ of finite explanations such that $K$ is
covering  then $\{w|w\in \mathcal{W}_T\wedge w\models A\}=\omega_K$ and $\mu(\{w|w\in \mathcal{W}_T\wedge w\models A\}=\mu(\omega_K)$ so $P(A)$ is well-defined.
In the case of Example \ref{sneezing_lpad}, $K_2$ shown in equation \ref{Ls} is a finite covering set of finite explanations for $moderate\_sneezing(david)$ that is mutually incompatible, so
\[P(moderate\_sneezing(david))=0.5\cdot 0.2+0.5\cdot 0.2+ 0.6=0.8.\]



\section{Dynamic Stratification of LPADs}
\label{mod-strat}

One of the most important formulations of stratification is that of
{\em dynamic} stratification.  \cite{Przy89d} shows that a program has
a 2-valued well-founded model iff it is dynamically stratified, so
that it is the weakest notion of stratification that is consistent
with the WFS.
As presented in~\cite{Przy89d}, dynamic stratification computes strata
via operators on \emph{3-valued interpretations} -- pairs of the form
$\langle Tr;Fa \rangle$, where $Tr$ and $Fa$ are subsets of the Herbrand
base $\cH_P$ of a normal program $P$.
\begin{definition} \label{def:lrdyn-ops}
  For a normal program $P$, sets $Tr$ and $Fa$ of ground atoms, and a
  3-valued interpretation $I$ we define
  \begin{description}
  \item[$True^P_I(Tr) =$]
    $\{A|A$ is not true in $I$;  and 
                        there is a clause
                        $B \leftarrow L_1,...,L_n$
                in $P$, a ground substitution $\theta$ such that
                $A = B\theta$ and for every $1 \leq i \leq n$ either
                $L_i\theta$ is true in $I$, or $L_i\theta \in Tr$\};
  \item[$False^P_I(Fa) =$] 
$\{A|A$ is not false in $I$; and for every
    clause $B \leftarrow L_1,...,L_n$ in $P$ and ground substitution
    $\theta$ such that $A = B\theta$ there is some $i$ $(1 \leq i \leq
    n)$ such that $L_i\theta$ is false in $I$ or $L_i\theta \in Fa\}$.
\end{description}
\end{definition}
%
\cite{Przy89d} shows that $True^P_I$ and $False^P_I$ are both
monotonic, and defines $\cT^P_I$ as the least fixed point of $True^P_I(\emptyset)$
and $\cF^P_I$ as the greatest fixed point of
$False^P_I(\cH_P)$~\footnote{Below, we will sometimes omit the program $P$ in
  these operators when the context is clear.}.
%
In words, the operator $\cT_I$ extends the interpretation $I$ to add
the new atomic facts that can be derived from $P$ knowing $I$; $\cF_I$
adds the new negations of atomic facts that can be shown false in $P$
by knowing $I$ (via the uncovering of unfounded sets).  An iterated
fixed point operator builds up dynamic strata by constructing
successive partial interpretations as follows.
\begin{definition}[Iterated Fixed Point and Dynamic Strata]
\label{def:IFP}
For a normal program $P$ let 

\begin{center}
$  \begin{array}{rcl}
          WFM_0 & = & \langle \emptyset ; \emptyset \rangle;      \\
 WFM_{\alpha+1} & = &       WFM_{\alpha} \cup
                                \langle \cT_{WFM_\alpha};\cF_{WFM_\alpha} \rangle; \\
     WFM_\alpha & = & \bigcup_{\beta < \alpha} WFM_\beta, \mbox{ for limit ordinal }\alpha.
  \end{array}
$
\end{center}

\noindent
  Let $WFM(P)$ denote the fixed point interpretation $WFM_\delta$,
  where $\delta$ is the smallest (countable) ordinal such that both
  sets $\cT_{WFM_\delta}$ and $\cF_{WFM_\delta}$ are empty.
 We refer to $\delta$ as the {\em depth} of program $P$.  The {\em
   stratum} of atom $A$, is the least ordinal $\beta$ such that $A \in
 WFM_{\beta}$ (where $A$ may be either in the true or false component of
 $WFM_{\beta}$).
\end{definition}
%
\cite{Przy89d} shows that the iterated fixed point $WFM(P)$ is in fact
the well-founded model and that any undefined atoms of the
well-founded model do not belong to any stratum -- i.e. they are not
added to $WFM_{\delta}$ for any ordinal $\delta$. Thus, a program is \emph{dynamically stratified} if every atom belongs to a stratum.

Dynamic stratification captures the order in which recursive
components of a program must be evaluated.  Because of this, dynamic
stratification is useful for modeling operational aspects of program
evaluation.  Fixed-order dynamic stratification~\cite{SaSW99}, used in
Section~\ref{algorithm}, models programs whose well-founded model can
be evaluated using a fixed literal selection strategy.  In this class,
the definition of $False^P_I(Fa)$ in Definition~\ref{def:lrdyn-ops} is
replaced by\footnote{Without loss of generality, we assume throughout
  that the fixed literal selection strategy is left-to-right as in
  Prolog.}:
%
\begin{description}
  \item[$False^P_I(F) =$] 
    $\{A|A$ is not false in $I$; and
      for every clause $B \leftarrow L_1,...,L_n$ in $P$ and ground
       substitution $\theta$ such that $A = B\theta$ there is some $i$
       $(1 \leq i \leq n)$ such that $L_i\theta$ is false in
       $I$ or $L_i\theta \in Fa$, {\bf {\em and for all $j$ $(1 \leq j \leq
       i-1)$, $L_j\theta$ is true in $I\}$}}.
\end{description}
%
\cite{SaSW99} describes how fixed-order dynamic stratification
captures those programs that a tabled evaluation can evaluate with a
fixed literal selection strategy (i.e. without the SLG operations of
{\sc simplification} and {\sc delay}).

\begin{example} \label{lrd-examp} 
  The following program has a 2-valued well-founded model and so is
  dynamically stratified, but does not belong to other stratification
  classes in the literature, such as local, modular, or weak
  stratification.  {\em
\begin{tabbing}
fooooo\==fooooooooooooooooooooooooooooooo\=ooooooooooooo\=\kill
\>         s $\leftarrow$ $\neg$s.                \> 
        s $\leftarrow$ $\neg$p, $\neg$q, $\neg$r. \\
\>         p $\leftarrow$ q, $\neg$r, $\neg$s.       \>
        q $\leftarrow$ r, $\neg$p.                
\>        r $\leftarrow$ p, $\neg$q.
\end{tabbing}
}
\noindent
{\em p}, {\em q}, and {\em r}\ all belong to stratum 0, while {\em s}
belongs to stratum 1.  Note that the above program also meets the
definition of fixed-order dynamically stratified as does the simple program
{\em
\begin{tabbing}
fooooo\==fooooooooooooooooooooooooooooooo\=ooooooooooooo\=\kill
\>        p $\leftarrow$ $\neg$ p. \>
          p.
\end{tabbing}
}
\noindent
which is not locally, modularly, or weakly stratified.  Fixed-order
stratification is more general than local stratification, and than
modular stratification (since modular stratified programs can be
decidably rearranged so that they have failing prefixes).  It is
neither more nor less general than weak stratification.
\end{example}
%
As seen by the above examples, fixed-order dynamic stratification is a
fairly weak property for a program to have.  The above definitions of
(fixed-order) dynamic stratification for normal programs can be
straightforwardly adapted to LPADs -- an LPAD $T$ is
\emph{(fixed-order) dynamically stratified} if each $w \in {\cW_T}$ is
(fixed-order) dynamically stratified.

\subsection{Conditions for Well\--Definedness of the Distribution Semantics} \label{sec:bounded}
When a given LPAD $T$ contains function symbols there are two reasons
why the distribution semantics may not be well-defined for $T$.
First, a world of $T$ may not have a two-valued well-founded model;
and second, $\cH_T$ may contain an atom that does not have a finite
set of finite explanations that is covering
(cf. Section~\ref{sem_fs}).  As noted in
Section~\ref{sec:function-free}, we consider only sound LPADs in this
paper and in this section address the problem of determining whether
$\cH_T$ may contain a atom that does not have a finite set of finite
explanations that is covering.

As is usual in logic programming, we assume that a program $P$ is
defined over a language with a finite number of function and constant
symbols.  Given such an assumption, placing an upper bound on the size
of terms in a derivation implies that the number of different terms in
a derivation must be finite -- and for certain methods of derivation,
such as tabled or bottom-up evaluations, that the derivation itself is
finite.

To motivate our definitions, consider the normal program
$T_{\mathit{inf}}$: 
{\em
\begin{tabbing}
fooooo\==fooooooooooooooooooooooooooooooo\=ooooooooooooo\=\kill 
\>     p(s(X)) $\leftarrow$ p(X). \> p(0). 
\end{tabbing}
}
\noindent
This program does not have a model with a finite number of true or
undefined atoms, and accordingly, there is no upper limit on the size
of atoms produced either in a bottom-up derivation of the program
(e.g. using the fixed-point characterization of
Definition~\ref{def:IFP}), or in a top-down evaluation of the query
{\em p(Y)}.  However, the superficially similar program,  $T_{\mathit{fin}}$:
{\em \begin{tabbing}
fooooo\==fooooooooooooooooooooooooooooooo\=ooooooooooooo\=\kill
\>  p(X) $\leftarrow$ p(f(X)).  \> p(0). 
\end{tabbing}
} 
\noindent
does have a model with a finite number of true and undefined atoms.
Of course, the model for the program does not have a finite number of
false atoms, but (default) false atoms are generally not explicitly
represented in derivations.  The model can in fact be produced by
various derivation techniques, such as an alternating fixed point
computation~\cite{VG89} based on sets of true and of true or undefined
atoms; or by tabling with term depth abstraction \cite{TaSa86}.

From the perspective of the distribution semantics consider
$T'_{\mathit{fin}}$, the extension of $T_{\mathit{fin}}$ with the clause
{\em \begin{tabbing}
fooooo\==fooooooooooooooooooooooooooooooo\=ooooooooooooo\=\kill
\>  q : 0.5 $\leftarrow$ p(X).
\end{tabbing}
} 
\noindent
and $T'_{inf}$, the similar extension of $T_{\mathit{inf}}$.  Recall from
Definition~\ref{prob_def_LPAD} that the probability of an atom $A$ in
an LPAD is defined as a probability measure that is constructed from
finite sets of finite composite choices: accordingly, the distribution
semantics for $A$ is well-defined if and only if it has a finite set
of finite explanations that is covering.  In $T'_{\mathit{fin}}$, {\em q} has
such a finite set of finite explanations that is covering, and so its
distribution semantics is well-defined.  However, in $T'_{\mathit{inf}}$, {\em
  q} does not have a finite set of finite explanations that is
covering, and so the distribution semantics is not well-defined for
{\em q}, even though every world of $T'_{\mathit{inf}}$  has a total
well-founded model.

The following definition captures these intuitions, basing the notion
of \boundedtermsize{} on the preceding definition of dynamic
stratification.

\begin{definition}[\Boundedtermsize{} Programs] \label{def:bts}
Let $P$ be the ground instantiation of a normal program. and $I,Tr
\subseteq \mathcal{H}_P$. Then an application of $True^P_I(Tr)$
(Definition~\ref{def:lrdyn-ops}) has the \emph{\boundedtermsize{}}
property if there is a integer $L$ such that the size of every ground
substitution $\theta$ used to produce an atom in $True^P_I(Tr)$ is
less than $L$.  $P$ itself has the \emph{\boundedtermsize{}} property
if every application of $True^P_I$ used to construct $WFM(P)$ has the
\boundedtermsize{} property with the same bound $L$.  Finally, an LPAD $T$ has the
\emph{\boundedtermsize{}} property if each world of $T$ has
the \boundedtermsize{} property.
\end{definition}

Note that $T_{inf}$ does not have the \boundedtermsize{} property, but
$T_{fin}$ does.  While determining whether a program $P$ is
\boundedtermsize{} is clearly undecidable in general, $T_{fin}$ shows
that $ground(P)$ need not be finite if $P$ is \boundedtermsize{}.
However, the model of $P$ may be characterized as follows\footnote{The
  proof of this and other theorems is given in the online Appendix to this
  paper.}.

\begin{theorem} \label{thm:finite}
Let $P$ be a normal program.  Then $WFM(P)$ has a finite number of
true atoms iff $P$ has the \boundedtermsize{} property.
\end{theorem}
%
Theorem~\ref{thm:finite} gives a clear model-theoretic
characterization of \boundedtermsize{} normal programs: note that if
$ground(P)$ is infinite, then $WFM(P)$ may have an infinite number of
false or undefined atoms.  In the context of LPADs, the
\boundedtermsize{} property ensures the well\--definedness of the
distribution semantics.

\begin{theorem} \label{thm:finite-lpad}
Let $T$ be a sound \boundedtermsize{} LPAD, and let $A \in
\mathcal{H}_T$.  Then $A$ has a finite set of finite explanations that
is covering.
\end{theorem}

\noindent
The proof of Theorem~\ref{thm:finite-lpad} is presented in the online
Appendix; here we indicate the intuition behind the proof.  First, we
note that it is straightforward to show that since each world of an
LPAD $T$ has a finite number of true atoms by
Theorem~\ref{thm:finite}, explanations are finite.  On the other hand,
showing that a query has a finite covering set of explanations is less
obvious, as $T$ could have an infinite number of worlds.  The proof
addresses this by showing that $T$ has a finite number of models, in
turn shown by demonstrating the existence of a bound $L_T$ on the
maximal size of any true atom in any world of $T$.  The existence of
$L_T$ is shown by contradiction by demonstrating that if no bound
existed, a world could be constructed that was not \boundedtermsize.
The idea is explained in the following example.

\comment{
This in turn follows from showing that if $T$ has \boundedtermsize{}
there is an integer, $L_T$, representing the maximal size of any atom
in the model of any world of $T$.  Since the language of $T$ if
finite, there are only a finite number of possible atoms whose size is
less than or equal to $L_T$, so that a finite set of covering
explanations must be constructible.  The existence of $L_T$ is shown
by contradiction, and the intuition can be seen from the following
program.}
\begin{example}
Consider the program
\begin{tabbing}
fooooo\==fooooooooooooooooooooooooooooooo\=ooooooooooooo\=\kill 
\>     $q:0.5 \vee p(f(X)):0.5\leftarrow p(X).$ \> $p(0).$
\end{tabbing}
\noindent
This program has an infinite number of finite models, which consist of
true atoms
\begin{tabbing}
fooooo\==fooooooooooooooooooooooooooooooo\=ooooooooooooo\=\kill 
\> $\{q,p(0)\},\{q,p(0),p(f(0))\},\{q,p(0),p(f(0)),p(f(f(0)))\},\ldots$
\end{tabbing}
depending on the selections made for instantiations of the first
clause, and so no finite bound $L_T$ exists for this program.  However
such a program also has a selection that gives rise to an infinite
model
\begin{tabbing}
fooooo\==fooooooooooooooooooooooooooooooo\=ooooooooooooo\=\kill 
\> $
\{p(0),p(f(0)),p(f(f(0))),p(f(f(f(0)))),\ldots\}$
\end{tabbing}
and so is not \boundedtermsize.  
\end{example}
Although \boundedtermsize{} programs have appealing properties, such
programs can make only weak use of function symbols.  For instance, a
program containing the Prolog predicate {\em member/2} would not be
\boundedtermsize{}, although as any Prolog programmer knows, a query
to {\em member/2} will terminate whenever the second argument of the
query is ground.  We capture this intuition with {\em
  \boundedtermsize{} queries}.  The definition of such queries relies
on the notion of an atom dependency graph, whose definition we state
for LPADs.

\begin{definition}[Atom Dependency Graph] \label{def:adg}
Let $T$ be a ground LPAD.  Then the {\em atom dependency graph} of $T$
is a graph $(V,E)$ such that $V = \mathcal{H}_T$ and an edge $(v_1,v_2) \in
E$ iff there is a clause $C \in T$ such that 
\begin{enumerate}
\item $(v_1:\alpha_1) \in head(C)$ and if $v_2 \mbox{ or } \neg v_2 \in body(C)$; or
\item $(v_1:\alpha_1), (v_2:\alpha_2) \in head(C)$.
\end{enumerate}
\end{definition}

\noindent
Definition~\ref{def:adg} includes dependencies among atoms in the head
of a disjunctive LPAD clause, similar to how dependencies are defined
in disjunctive logic programs.  Given a ground LPAD $T$, the atom
dependency graph of $T$ is used to bound the search space of a
(relevant) derivation in a  world of $T$ under the WFS.

\begin{definition}[\Boundedtermsize{} Queries] \label{def:bts-queries}
Let $T$ be a ground LPAD, and $Q$ an atomic query to $T$ (not
necessarily ground).  Then the {\em atomic search space} of $Q$
consists of the union of all ground instantiations of $Q$ in
$\mathcal{H}_T$ together with all atoms reachable in the atom
dependency graph of $T$ from any ground instantiation of $Q$.  Let
\[
T_Q = \{C | C \in T \mbox{ and a head atom of $C$ is in the atomic search space of } Q \}
\]
\noindent
The query $Q$ is \emph{\boundedtermsize{}} if $T_Q$ is a \boundedtermsize{}
program.
\end{definition}

\noindent
The notion of a bounded-term size query will be used in
Section~\ref{tabling} to characterize termination of the SLG tabling
approach, and in Section~\ref{correctness} to characterize correctness
and termination of our tabled PITA implementation.


%


\subsection{Comparisons of Termination Properties} \label{sec:compare}

We next consider how the concepts of \boundedtermsize{} programs and
queries relate to some other classes of programs for which termination
has been studied.  Since the definitions of the previous section are
based on LPADs, and other work in the literature is often based on
disjunctive logic programs, we restrict our attention to normal
programs, for which the semantics coincide. 

\cite{BaBC09} studies the class of {\em finitely recursive} programs,
which is a superset of {\em finitary} programs previously introduced
into the literature by the authors.  The paper first defines a
dependency graph, which for normal programs is essentially the same as
Definition~\ref{def:adg}.  A finitely recursive normal program, then,
is one for which in its atom dependency graph, only a finite number of
vertices are reachable from any vertex.  It is easy to see that
neither \boundedtermsize{} programs nor finitely recursive programs
are a subclass of each other.  A program containing simply {\em
  member/2} (and a constant) is finitely recursive, but is not
\boundedtermsize{}.  However, the program
{\em \begin{tabbing}
fooooo\==fooooooooooooooooooooooooooooooo\=ooooooooooooo\=\kill
\>  p(X) $\leftarrow$ p(f(X)). 
\end{tabbing}
} 
\noindent
has \boundedtermsize{}, as does the program
{\em 
\begin{tabbing}
fooooo\==fooooooooooooooooooooooooooooooo\=ooooooooooooo\=\kill \>
    p(s(X)) $\leftarrow$ q(X),p(X). \> p(0).
\end{tabbing}
}
\noindent
although neither is finitely recursive (for the last program, the
failure of {\em q(X)} means that all applications of $True_I$ have
\boundedtermsize{}).  However, note that for any program $P$ that is
finitely recursive, all ground atomic queries to $P$ will have
\boundedtermsize{}.  Therefore, if $P$ is finitely recursive, every
ground atomic query to $P$ will be \boundedtermsize{}, even if $P$
itself isn't \boundedtermsize{}.

Another recent work \cite{CCIL08} defines the class {\em
  finitely-ground} programs.  We do not present its formalism here,
but Corollary 1 of \cite{CCIL08} states that if a program is
finitely-ground, it will have a finite number of answer sets and each
answer set will be finite (as represented by the set of true atoms in
the model).  By Theorem~\ref{thm:finite} of this paper, such a program
will have \boundedtermsize{}, so that finitely-ground programs may be
co-extensive with \boundedtermsize{} programs.  On the other hand,
\cite{CCIL08} notes that finitely-ground programs and finitely
recursive programs are incompatible.  Non-range restricted programs
are not finitely-ground, although they can be finitely recursive. As
discussed above, any ground atomic query to a finitely recursive
program will have \boundedtermsize{}, so that finitely-ground programs
must be a proper subclass of those programs for which all ground
atomic queries have \boundedtermsize{}.

To summarize for normal programs:
\begin{itemize}
\item  Finitely recursive and \boundedtermsize{} programs are incompatible,
but finitely recursive programs are a proper subclass of those
programs for which all ground atomic queries are \boundedtermsize{}.

\item Finitely-ground and \boundedtermsize{} programs appear to be
  co-extensive, but finitely-ground programs are a proper subclass of those
programs for which all ground atomic queries are \boundedtermsize{}.
\end{itemize}

\section{Representing Explanations by Means of Decision Diagrams}
\label{bdd}
In order to represent explanations we can use Multivalued Decision Diagrams (MDDs) \cite{ThaDav-MDD-78}.
An MDD represents a  function $f(\mathbf{X})$ taking Boolean values on a set of multivalued variables ${\mathbf X}$ by means of  
a rooted graph that has one level for each variable. Each node $N$ has one child for each possible value of the multivalued variable associated to $N$. The leaves store either 0 or 1.
Given values for all the variables ${\mathbf X}$, an MDD can be used to compute the value of $f(\mathbf{X})$  by traversing the graph starting from the root and returning the value associated to the leaf that is reached.

Given a set of explanations $K$, we obtain a Boolean function $f_K$ in
the following way. Each ground clause $C\theta$ appearing in $K$ is
associated to a multivalued variable $X_{C\theta}$ with as many values
as atoms in the head of $C$. In other words, each atomic choice
$(C,\theta,i)$ is represented by the propositional equation
$X_{C\theta}=i$.\ Equations for a single explanation are conjoined and
the conjunctions for the different explanations are disjoined.
The set of explanations in Equation (\ref{expls})
can be represented by the function
$f_{K_1}(\mathbf{X})=(X_{C_1\{X/david\}}=2)\vee (X_{C_2\{X/david\}}=2)$. 
The MDD shown in Figure \ref{MDD} represents $f_{K_1}(\mathbf{X})$.

\begin{figure}
\centering
\subfigure
	[MDD.\label{MDD}]	{$$
\xymatrix@C=0.01mm@R=1.2mm{ & *+<3pt>[F-:<3pt>]{\textcolor{white}{ciao}}
\ar@/^/@{-}[r]^1\ar@/_/@{-}[r]^3 \ar@/_/@{-}[ddr]^2
& *+<3pt>[F]{\textcolor{white}{a}0\textcolor{white}{a}}\\
 *+<3pt>[F-:<3pt>]{\textcolor{white}{ciao}} 
\ar@/_/@{-}[drr]^2\ar@/^/@{-}[ur]^1 \ar@/_/@{-}[ur]^3\\
&& *+<3pt>[F]{\textcolor{white}{a}1\textcolor{white}{a}}\\
X_{C_1\{X/david\}}&X_{C_2\{X/david\}}
}
$$
}
\hspace{0.2cm}
\subfigure
	[BDD.\label{BDD}]	{$$
\xymatrix@C=0.01mm@R=1.4mm{ & & *+<3pt>[F-:<3pt>]{\textcolor{white}{ciao}}
\ar@/_/@{-}[dr]^0\ar@/^1pc/@{-}[rr]^1 
& &*+<3pt>[F]{\textcolor{white}{a}0\textcolor{white}{a}}\\
*+<3pt>[F-:<3pt>]{\textcolor{white}{ciao}} 
\ar@/^0.5pc/@{-}[urr]^1\ar@/_/@{-}[dr]^0 
&&& *+<3pt>[F-:<3pt>]{\textcolor{white}{ciao}} 
\ar@/_/@{-}[dr]^1\ar@/^/@{-}[ur]^0 \\
& *+<3pt>[F-:<3pt>]{\textcolor{white}{ciao}} 
\ar@/_0.5pc/@{-}[rrr]^1\ar@/^/@{-}[uur]^0
&&&*+<3pt>[F]{\textcolor{white}{a}1\textcolor{white}{a}}
 \\
X_{C_1\{X/david\}1}&X_{C_1\{X/david\}2}&X_{C_2\{X/david\}1}&X_{C_2\{X/david\}2}
}
$$
}
\caption{Decision diagrams for Example  \ref{sneezing_lpad}.}
\label{dd}
\end{figure}
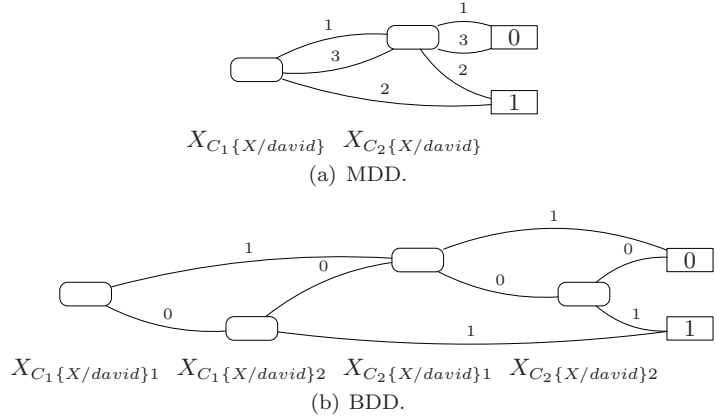

Given a MDD $M$, we can identify a set of explanations $K_M$ associated to $M$ that is obtained by considering each path from the root to a 1 leaf as an explanation. It is easy to see that if $K$ is a set of explanations and $M$ is obtained from $f_K$, $K$ and $K_M$ represent the same set of worlds, i.e., that $\omega_K=\omega_{K_M}$.

Note that $K_M$ is mutually incompatible because at each level we
branch on a variable so that the explanations associated to the leaves
that are descendants of a child of a node $N$ are incompatible with
those of any other children of $N$.

By converting a set of explanations into a mutually incompatible set
of explanations, MDDs allow the computation of $\mu(\omega_K)$
(Section~\ref{sem_fs}) given any $K$. This is equivalent to computing
the probability of a DNF formula which is \#P-complete
\cite{DBLP:journals/siamcomp/Valiant79}. Decision diagrams offer a
practical solution for this problem and were shown better than other methods
\cite{DBLP:conf/ijcai/RaedtKT07}.

Decision diagrams can be built with various software packages 
that provide highly efficient implementation of Boolean operations. However, most packages are restricted to work with Binary Decision Diagrams, i.e., decision diagrams where all the variables are Boolean.
 To manipulate MDDs with a BDD package, we must represent multivalued variables by means of binary variables. Various options are possible, we found that the following, proposed in \cite{DeR-NIPS08},
gives the best performance. For a variable $X$ having $n$ values, we use $n-1$ Boolean variables $X_{1},\ldots,X_{n-1}$ and we represent the equation $X=i$ for $i=1,\ldots n-1$ by means of the conjunction $$\overline{X_{1}}\wedge\overline{X_{2}}\wedge\ldots \wedge \overline{X_{i-1}}\wedge X_{i}$$
 and the equation $X=n$ by means of the conjunction $$\overline{X_{1}}\wedge\overline{X_{2}}\wedge\ldots \wedge\overline{X_{n-1}}$$
The BDD representation of the function $f_{K_1}$ is given in Figure \ref{BDD}.
The Boolean variables are associated with the following parameters: 
$$\begin{array}{ll}
P(X_{1})=P(X~=~1)\\
\ldots\\
 P(X_{i})=\frac{P(X=i)}{\prod_{j=1}^{i-1}(1-P(X_{j-1}))}
\end{array}$$.

%
%
%
%
%
%
%
%

\section{Tabling and Answer Subsumption}
\label{tabling}
The idea behind tabling is to maintain in a table both subgoals
encountered in a query evaluation and answers to these subgoals.  If a
subgoal is encountered more than once, the evaluation reuses
information from the table rather than re-performing resolution
against program clauses.  Although the idea is simple, it has
important consequences.  First, tabling ensures termination for a wide
class of programs, and it is often easier to reason about termination
in programs using tabling than in basic Prolog.  Second, tabling can
be used to evaluate programs with negation according to the WFS.
Third, for queries to wide classes of programs, such as datalog
programs with negation, tabling can achieve the optimal complexity for
query evaluation.  And finally, tabling integrates closely with
Prolog, so that Prolog's familiar programming environment can be used,
and no other language is required to build complete systems.  As a
result, a number of Prologs now support tabling including XSB, YAP,
B-Prolog, ALS, and Ciao.  In these systems, a predicate $p/n$ is
evaluated using SLDNF by default: the predicate is made to use tabling
by a declaration such as {\em table p/n} that is added by the user or
compiler.

This paper makes use of a tabling feature called {\em answer
  subsumption}.  Most formulations of tabling add an answer $A$ to a
table for a subgoal $S$ only if $A$ is a not a variant (as a term) of
any other answer for $S$.  However, in many applications it may be
useful to order answers according to a partial order or (upper
semi-)lattice.  As an example, consider the case of a lattice on the
values of the second argument of a binary predicate $p/2$. Answer
subsumption may be specified by a declaration such as {\em table
  p(\_,or/3 - zero/1).}
%
%
where $zero/1$ is the bottom element of the lattice and $or/3$ is the
join operation of the lattice. For example, if a table had an answer
$p(a,b_1)$ and a new answer $p(a,b_2)$ were derived, the answer
$p(a,b_1)$ is replaced by $p(a,b_3)$, where $b_3$ is the join of $b_1$
and $b_2$ obtained by calling $or(b_1,b_2,b_3)$. In the PITA
algorithm for LPADs presented in Section \ref{algorithm} the last
argument of an atom is used to store explanations for the atom in the
form of BDDs and the $or/3$ operation is the logical disjunction of
two explanations~\footnote{The logical disjunction $b_3$ can be seen as
  subsuming $b_1$ and $b_2$ over the partial order af implication
  defined on propositional formulas that represent explanations.}.
Answer subsumption over arbitrary upper semi-lattices is implemented
in XSB for stratified programs~\cite{Swif99a}.

For formal results in this section and Section~\ref{correctness} we
use SLG resolution~~\cite{DBLP:journals/jacm/ChenW96}, under the
forest-of-trees representation~\cite{Swif99b}; this framework is
extended with answer subsumption in the proof of
Theorem~\ref{eval-s-n-c}.  However, first we present a theorem stating
that \boundedtermsize{} queries (Definition~\ref{def:bts-queries}) to
normal programs are amenable to top-down evaluation using tabling.
Although SLG has been shown to finitely terminate for other notions of
\boundedtermsize{} queries, the concept as presented in
Definition~\ref{def:bts-queries} is based on a bottom-up fixed-point
definition of WFS, and only bounds the size of substitutions used in
$True^P_I$ of Definition~\ref{def:lrdyn-ops}, but not of $False^P_I$.
In fact, to prove termination of SLG with respect to
\boundedtermsize{} queries, SLG must be extended so that its {\sc New
  Subgoal} operation performs what is called {\em term-depth
  abstraction}~\cite{TaSa86}, explained informally as follows.  An SLG
evaluation can be formalized as a forest of trees in which each tree
corresponds to a unique (up to variance) subgoal.  The SLG {\sc New
  Subgoal} operation checks to see if a given selected subgoal $S$ is
the root of any tree in the current forest.  If not, then a new tree
with root $S$ is added to the forest.  Without term-depth abstraction,
an SLG evaluation of the query {\em p(a)} and the program consisting
of the single clause
{\em \begin{tabbing}
fooooo\==fooooooooooooooooooooooooooooooo\=ooooooooooooo\=\kill
\>  p(X) $\leftarrow$ p(f(X)). 
\end{tabbing}
} 
\noindent
would create an infinite number of trees.  However, if the {\sc New
  Subgoal} operation uses term-depth abstraction, any subterm in $S$
over a pre-specified maximal depth would be replaced by a new
variable.  For example, in the above program if the maximal depth were
specified as 3, the subgoal {\em p(f(f(f(a))))} would be rewritten to
{\em p(f(f(f(X))))} for the purposes of creating a new tree.  The
subgoal {\em p(f(f(f(a))))} would consume any answer from the tree for
{\em p(f(f(f(X))))} where the binding for $X$ unified with $a$.  In
this manner it can be ensured that only a finite number of trees were
created in the forest.  This fact, together with the size bound on the
derivation of answers provided by Definition~\ref{def:bts-queries}
ensures the following theorem, where a finitely terminating evaluation
may terminate normally or may terminate through floundering.

\begin{theorem} \label{thm:normal-tab-term}
Let $P$ be fixed-order dynamically stratified normal program, and $Q$
a \boundedtermsize{} query to $P$.  Then there is an SLG evaluation of
$Q$ to $P$ using term-depth abstraction that finitely terminates.
\end{theorem}
By the discussion of Section~\ref{sec:compare},
Theorem~\ref{thm:normal-tab-term} shows that there is an SLG
evaluation with term-depth abstraction will finitely terminate on any
ground query to a finitely recursive~\cite{BaBC09} or
finitely-ground~\cite{CCIL08} program that is fixed-order
stratified~\footnote{The proof of Theorem~\ref{thm:normal-tab-term}
  relies on a delay-minimal evaluation of $Q$ that does not produced
  any conditional answers -- that is, an evaluation that does not
  explore the space of atoms that are undefined in $WFM(P)$.}.  While
SLG itself is ideally complete for all normal programs, the PITA
implementation is restricted to fixed-order stratified programs, so
that Theorem~\ref{thm:normal-tab-term} is used in the proof of the
termination results of Section~\ref{correctness}.

\section{Program Transformation}
\label{algorithm}
The first step of the PITA algorithm is to apply a program
transformation to an LPAD to create a normal program that
contains calls for manipulating BDDs.  In our implementation, these
calls provide a Prolog interface to the
CUDD\footnote{\url{http://vlsi.colorado.edu/~fabio/}} C
library and use the following predicates\footnote{BDDs are
  represented in CUDD as pointers to their root node.}
\begin{itemize}
  \item \textit{init, end}: for allocation and deallocation of a BDD manager, a data structure used to keep track of the memory for storing BDD nodes;
  \item \textit{zero(-BDD), one(-BDD), and(+BDD1,+BDD2,-BDDO), or(+BDD1,+BDD2,\\ -BDDO), not(+BDDI,-BDDO)}: Boolean operations between BDDs;
  \item \textit{add\_var(+N\_Val,+Probs,-Var)}: addition of a new multi-valued variable with \textit{N\_Val} values and parameters \textit{Probs};
  \item \textit{equality(+Var,+Value,-BDD)}: \textit{BDD} represents \textit{Var=Value}, i.e. that the random variable {\em Var} is assigned {\em Value} in the BDD;
  \item \textit{ret\_prob(+BDD,-P)}: returns the probability of the formula encoded by \textit{BDD}.
  \end{itemize}
%
%
\textit{add\_var(+N\_Val,+Probs,-Var)} adds a new random variable associated to a new  instantiation of a rule with \textit{N\_Val} head atoms and parameters list \textit{Probs}.
The PITA transformation uses the auxiliary  predicate \textit{get\_var\_n(+R,+S,+Probs,-Var)} to wrap \textit{add\_var/3} and avoid adding a new variable when one already exists for an instantiation. As shown below, a new fact \textit{var(R,S,Var)} is asserted each time a new random variable is created, where \textit{R} is an identifier for the LPAD clause, \textit{S} is a list of constants, one for each variable of the clause, and \textit{Var} is an integer that identifies the random variable associated with clause \textit{R} under the grounding represented by \textit{S}. The auxiliary predicate has the following definition

\[\begin{array}{ll}
get\_var\_n(R,S,Probs,Var)\leftarrow\\
\ \ \   (var(R,S,Var)\rightarrow true;\\
\ \ \ \    length(Probs,L), add\_var(L,Probs,Var), assert(var(R,S,Var))).
\end{array}\]

\noindent
%
The PITA transformation applies to atoms, literals and clauses.
  If $H$ is an atom, $PITA_H(H)$ is $H$ with
the variable $BDD$ added as the last argument. 
 If $A_j$ is an atom,
$PITA_B(A_j)$ is $A_j$ with the variable $B_j$ added as the last
argument. 
In either case for an atom $A$, $BDD(PITA(A))$ is the value of the last
argument of $PITA(A)$, 
 If $L_j$ is negative literal $\neg A_j$, $PITA_B(L_j)$ is the conditional 
\[(PITA_B'(A_j)\rightarrow not(BN_j,B_j);one(B_j)),\]
 where $PITA_B'(A_j)$ is $A_j$ with the variable $BN_j$ added as the last argument.  
In other words the input BDD, $BN_j$, is negated if it exists; otherwise the BDD for the constant function $1$ is returned.

 A non-disjunctive fact $C_r=H$ is transformed into the clause  

\[PITA(C_r)=\ \ PITA_H(H)\leftarrow one(BDD).\]

\noindent A disjunctive fact $C_r=H_1:\alpha_1\vee \ldots\vee H_n:\alpha_n$.
where the parameters sum to 1, is transformed into the set of clauses $PITA(C_r)$\footnote{The second argument of $get\_var\_n$ is the empty list because a fact does not contain variables since the program is \boundedtermsize{}.}

\[\begin{array}{lll}
PITA(C_r,1)=&PITA_H(H_1)\leftarrow &get\_var\_n(r,[],[\alpha_1,\ldots,\alpha_n],Var),\\
&&equality(Var,1,BDD).\\
\ldots\\
PITA(C_r,n)=&PITA_H(H_n)\leftarrow &get\_var\_n(r,[],[\alpha_1,\ldots,\alpha_n],Var),\\
&&equality(Var,n,BDD).\\
\end{array}\]

\noindent
In the case where the parameters do not sum to one, the clause is first transformed into
$H_1:\alpha_1\vee \ldots\vee H_n:\alpha_n\vee null:1-\sum_{1}^n\alpha_i.$
and then into the clauses above, where the list of parameters is $[\alpha_1,\ldots,\alpha_n,1-\sum_{1}^n\alpha_i]$ but the $(n+1)$-th clause (the one for $null$)
is not generated. 

\noindent
The definite clause $C_r=H\leftarrow L_1,\ldots,L_m$.
is transformed into the clause 

\[\begin{array}{lll} 
PITA(C_r)=& PITA_H(H)\leftarrow &one(BB_0),\\
&&PITA_B(L_1),and(BB_0,B_1,BB_1),\\
&&\ldots,\\
&&PITA_B(L_m),and(BB_{m-1},B_m,BDD).
\end{array}\]

\noindent
The disjunctive clause

\[C_r=H_1:\alpha_1\vee \ldots\vee H_n:\alpha_n\leftarrow L_1,\ldots,L_m.\]

\noindent
where the parameters sum to 1, is transformed into the set of clauses $PITA(C_r)$

\[\begin{array}{lll} 
PITA(C_r,1)=& PITA_H(H_1)\leftarrow&one(BB_0),\\
&&PITA_B(L_1),and(BB_0,B_1,BB_1),\\
&&\ldots,\\
&&PITA_B(L_m),and(BB_{m-1},B_m,BB_m),\\
&&get\_var\_n(r,VC,[\alpha_1,\ldots,\alpha_n],Var),\\
&&equality(Var,1,B),and(BB_m,B,BDD).\\
\ldots\\
PITA(C_r,n)=&PITA_H(H_n)\leftarrow&one(BB_0),\\
&&PITA_B(L_1),and(BB_0,B_1,BB_1),\\
&&\ldots,\\
&&PITA_B(L_m),and(BB_{m-1},B_m,BB_m),\\
&&get\_var\_n(r,VC,[\alpha_1,\ldots,\alpha_n],Var),\\
&&equality(Var,n,B),and(BB_m,B,BDD).
\end{array}\]

\noindent
where  $VC$ is a list containing each variable appearing in $C_r$.
If the parameters do not sum to 1, the same technique used for disjunctive facts is used.

\begin{example}
\label{ex-tra}
Clause $C_1$ from the LPAD of Example \ref{sneezing_lpad} is translated into

\[
 \begin{array}{llll}
 strong\_sneezing(X,BDD)&\leftarrow one(BB_0), \mathit{flu}(X,B_1),and(BB_0,B_1,BB_1),\\
 &get\_var\_n(1,[X],[0.3,0.5,0.2],Var),\\
 &equality(Var,1,B),and(BB_1,B,BDD).\\
 moderate\_sneezing(X,BDD)&\leftarrow one(BB_0), \mathit{flu}(X,B_1),and(BB_0,B_1,BB_1),\\
 &get\_var\_n(1,[X],[0.3,0.5,0.2],Var),\\
&equality(Var,2,B),and(BB_1,B,BDD).
\end{array}
\]

\noindent
while 
clause $C_3$ is translated into

\[
 \begin{array}{llll}
 flu(david,BDD)&\leftarrow one(BDD).
\end{array}
\]
\end{example}
In order to answer queries, the goal {\em prob(Goal,P)} is used, which is defined by 

\[
 \begin{array}{llll}
prob(Goal,P)&\leftarrow   init, retractall(var(\_,\_,\_)), \\
&add\_bdd\_arg(Goal,BDD,GoalBDD),\\
\ \ \   &(call(GoalBDD)\rightarrow ret\_prob(BDD,P) ; P=0.0),\\
&end.
\end{array}
\]

\noindent
where $add\_bdd\_arg(Goal,BDD,GoalBDD)$ implements $PITA_H(Goal)$.
Moreover, various predicates of the LPAD should be declared as tabled. For a predicate $p/n$, the declaration is
{\em table p(\_1,...,\_n,or/3-zero/1)},
which indicates that answer subsumption
is used to form the disjunct of multiple explanations. At a minimum,
the predicate of the goal and all the predicates appearing in negative literals should be tabled with answer subsumption. As shown in Section \ref{exp}, it is usually better to table every predicate whose answers have multiple explanations and are going to be reused often.

\section{Correctness of PITA Evaluation}~\label{correctness}
\noindent
In this section we show a result regarding the PITA transformation and
its tabled evaluation on \boundedtermsize{} queries: this result takes
as a starting point the well\--definedness result of
Theorem~\ref{thm:finite-lpad}.



The main result of this section, Theorem~\ref{eval-s-n-c}, makes
explicit mention of BDD data structures, which are considered to be
ground terms for the purposes of formalization and are not specified
further.  Accordingly, the BDD operations used in the PITA
transformation: {\em and/3}, {\em or/3}, {\em not/2}, {\em one/1},
{\em zero/1}, and {\em equality/3}, are all taken as (infinite)
relations on terms, so that these predicates can be made part of a
program's ground instantiation in the normal way.  As a result, the
ground instantiation of $PITA(T)$ instantiates all variables in $T$
with all BDD terms.  Similarly, for the purposes of proving
correctness, a ground program is assumed to be extended with the
relation \textit{var(RuleName,[],Var)} to associate a random variable
with the identifier of each clause (see Appendix C for more details).
Note that since Theorem~\ref{eval-s-n-c} assumes a \boundedtermsize{}
query, the semantics is well-defined so the BDD and {\em var/3} terms
are finite. In other words, the representation of each explanation of
each atom are finite, and each atom has a finite covering set of
explanations.
Lemma \ref{lem:pita-bts} shows that the PITA transformation does not
affect the property of a query being \boundedtermsize. a result that
is used in the proof of Theorem~\ref{eval-s-n-c}.
\begin{lemma} \label{lem:pita-bts}
Let $T$ be an LPAD and $Q$ a \boundedtermsize{} query to $T$.  Then
the query $PITA_H(Q)$ to $PITA(T)$ has \boundedtermsize{}.
\end{lemma}
Theorem~\ref{eval-s-n-c} below states the correctness of the tabling
implementation of PITA, since the BDD returned for a tabled query is
the disjunction of a covering set of explanations for that query.  The
proof uses an extension of SLG evaluation that includes answer
subsumption to collect explanations by disjoining BDDs, but that is
restricted to the fixed-order dynamically stratified programs of
Section~\ref{mod-strat}.  This formalism models the programs and
implementation tested in Section~\ref{exp}.
%
%
\begin{theorem}[Correctness of PITA Evaluation] \label{eval-s-n-c} 
Let $T$ be a fixed-order dynamically stratified LPAD and $Q$ a ground
\boundedtermsize{} atomic query. Then there is an SLG evaluation $\cE$ of
$PITA_H(Q)$ against $PITA(T_Q)$, such that answer subsumption is
declared on $PITA_H(Q)$ using BDD-disjunction where $\cE$ finitely
terminates with an answer $Ans$ for $PITA_H(Q)$ and $BDD(Ans)$
represents a covering set of explanations for $Q$.
\end{theorem}

\section{Related Work}
\label{related}
 \cite{DBLP:conf/iclp/MantadelisJ10} presented an algorithm for answering queries to ProbLog programs that uses tabling. Our work differs from this in two important ways. The first is that we use directly XSB tabling with
answer subsumption while  \cite{DBLP:conf/iclp/MantadelisJ10}  use some user-defined predicates
that manipulate  extra tabling data structures.  The second
difference is that in  \cite{DBLP:conf/iclp/MantadelisJ10} explanations are stored in trie data structures
that are then translated into BDDs.  When
translating the tries into BDDs, the algorithm of \cite{DBLP:conf/iclp/MantadelisJ10} finds shared
substructures, i.e., sub-explanations shared by many explanations. By
identifying shared structures the construction of BDDs is sped up
since sub-explanations are transformed into BDD only once.  In our
approach, we similarly exploit the repetition of structures but we do
it while finding explanations: by storing in the table the BDD
representation of the explanations of each answer,  every time the
answer is reused its BDD does not have to be rebuilt. Thus our
optimization is guided by the derivation of the query. Moreover, if a
BDD is combined with another BDD that already contains the first as a
subgraph, we rely on the highly optmized CUDD functions for the
identification of the repetition and the simplification of the
combining operation.  In this way we exploit structure sharing as well
without the intermediate pass over the trie data strucutres.

\section{Experiments}
\label{exp}
\algorithm\ was tested on two datasets that contain function symbols: the first is taken from \cite{VenVer04-ICLP04-IC} and encodes a Hidden Markov Model (HMM) while the second from \cite{DBLP:conf/ijcai/RaedtKT07} encodes biological networks. Moreover, it was also tested on the four testbeds of \cite{MeeStrBlo08-ILP09-IC} that do not contain function symbols. \algorithm\ was compared with the exact version of ProbLog
 \cite{DBLP:conf/ijcai/RaedtKT07}  available in the git version of Yap as of 10 November 2010, with the version of \texttt{cplint} \cite{Rig-AIIA07-IC} available in Yap 6.0 and with the version of CVE \cite{MeeStrBlo08-ILP09-IC} available in ACE-ilProlog 1.2.20\footnote{All experiments were performed on Linux machines with an Intel Core 2 Duo E6550 (2333 MHz) processor and 4 GB of RAM.}.
%

The first problem  models a hidden Markov model with  states 1, 2 and 3, of which 3 is an end state. This problem is encoded by the program

\begin{tabbing}
fooooo\==fooooo\=oooooooooooooooooooooooooo\=ooooooooooooo\=\kill 
\> \textit{s(0,1):1/3 $\vee$ s(0,2):1/3 $\vee$ s(0,3):1/3.} \\
\> \textit{s(T,1):1/3 $\vee$  s(T,2):1/3 $\vee$  s(T,3):1/3 $\leftarrow$} \\
\> \> \textit{T1 is  T-1, T1$>$=0, s(T1,F), $\backslash$+ s(T1,3).}
\end{tabbing}

\noindent 
For this experiment, we query the probability of the HMM being in state 1 at time \textit{N}
for increasing values of \textit{N}, i.e., we query the probability of \textit{s(N,1)}.
In \algorithm\  and ProbLog, we did not use reordering of BDDs variables\footnote{For each experiment with PITA and ProbLog, we used either group sift automatic reordering or no reordering of BDDs variables depending on which gave the best results.}. In PITA we tabled $on/2$ and in ProbLog  we tabled the same predicate using the technique described in \cite{DBLP:conf/iclp/MantadelisJ10}.
The execution times of \algorithm , ProbLog, CVE and \texttt{cplint} are shown in Figure \ref{die}.
In this problem tabling provides an impressive speedup, since computations can be reused often.

\begin{figure}
\begin{center}
\includegraphics[width=.55\textwidth]{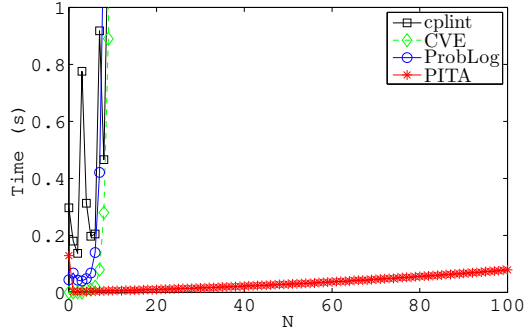}
\end{center}
\caption{Hidden Markov model.}
\label{die}
\end{figure}
The biological network programs compute the probability of a path in a large graph in which the nodes encode biological entities and the links represents conceptual relations among them. Each program in this dataset contains a non\--probabilistic definition of path plus a number of links represented by probabilistic facts. The programs have been sampled from a very large graph and contain 200, 400, $\ldots$, 10000 edges. 
Sampling was repeated ten times, to obtain ten series of programs of increasing size. In each program we query the probability that the two genes HGNC\_620 and HGNC\_983 are related.
We used two definitions of path. The first, from \cite{ProbLog-impl}, performs loop checking explicitly by keeping the list of visited nodes:

\begin{equation}\begin{array}{lll}
\label{loop}
path(X,Y)&\leftarrow &path(X,Y,[X],Z).\label{p}\\
path(X,Y,V,[Y|V])&\leftarrow &arc(X,Y).\label{p41}\\                 
path(X,Y,V0,V1)&\leftarrow &arc(X,Z),append(V0,\_S,V1),\\
&&\backslash + member(Z,V0),path(Z,Y,[Z|V0],V1).\label{p42}\\
arc(X,Y)&\leftarrow &edge(X,Y).\\
arc(X,Y)&\leftarrow &edge(Y,X).
\end{array}
\end{equation}
The second exploits tabling for performing loop checking:
\begin{equation}\begin{array}{lll}
\label{tab}
path(X,X).\\
path(X,Y,)&\leftarrow &path(X,Z),arc(Z,Y).\\
arc(X,Y)&\leftarrow &edge(X,Y).\\
arc(X,Y)&\leftarrow &edge(Y,X).
\end{array}
\end{equation}
The possibility of using  lists (that require function symbols) allowed in this case more modeling freedom. In PITA, the predicates $path/2$, $edge/2$ and $arc/2$ are tabled in both cases.
For ProbLog we used its implementation of tabling for loop checking in the second program. As in PITA,  $path/2$, $edge/2$ and $arc/2$ are tabled.
\begin{figure}
\centering
\subfigure
	[Number of successes.\label{graph_succ_8}]	{\includegraphics[width=.47\textwidth]{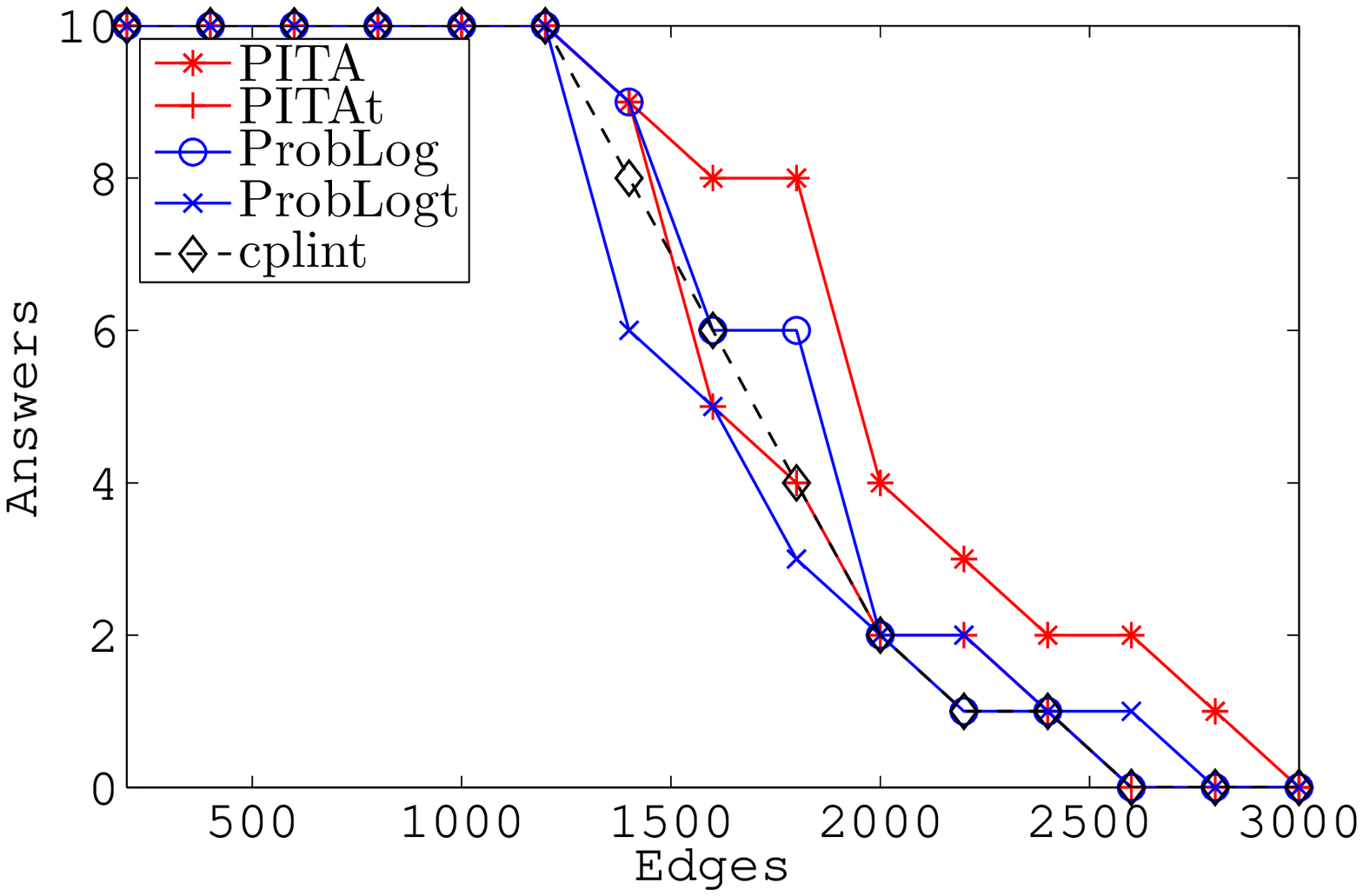}}
\hspace{0.2cm}
\subfigure
	[Average execution times on the graphs on which all the algorithm succeeded.\label{graph_times_8}]	{\includegraphics[width=.47\textwidth]{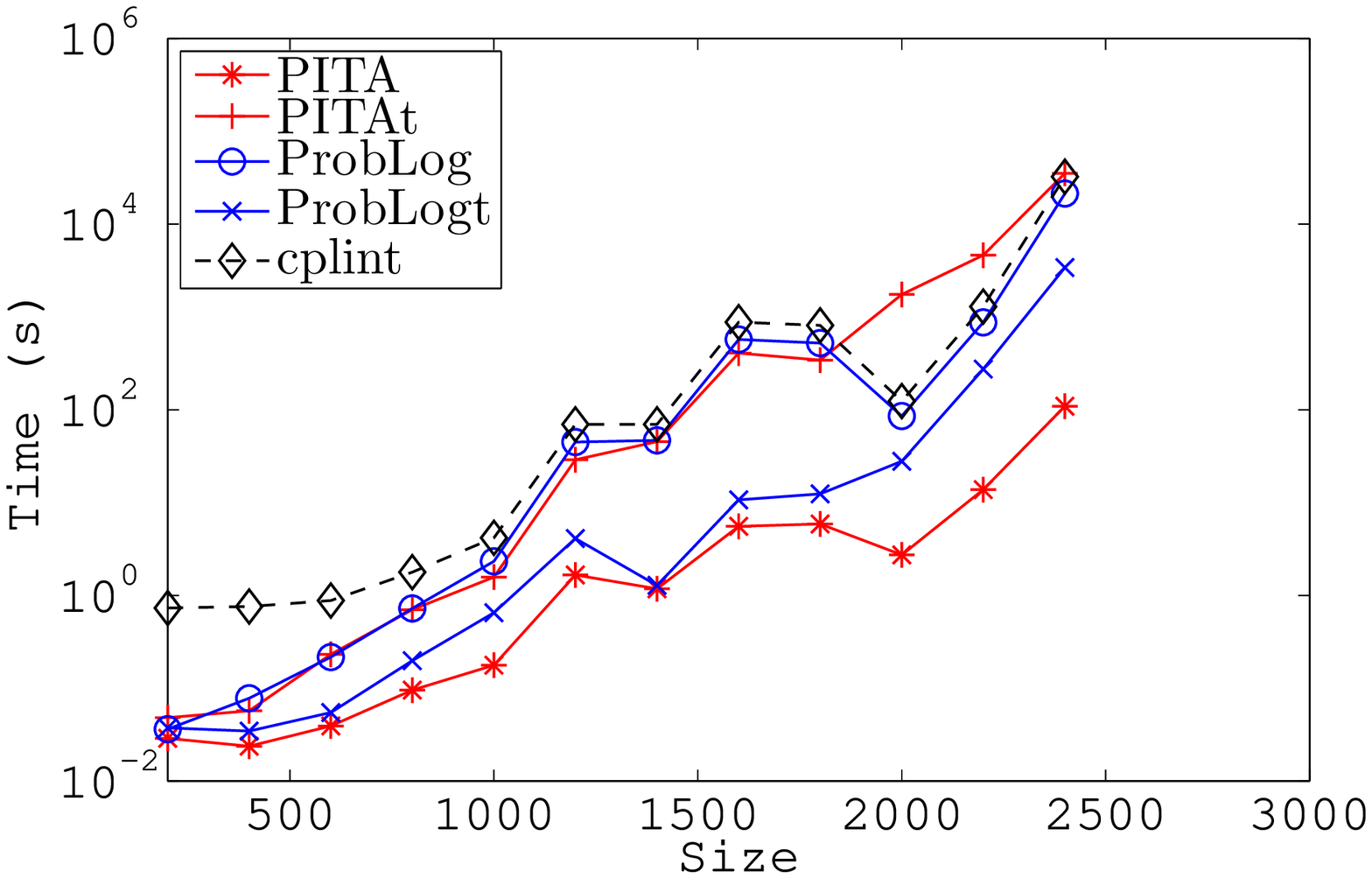}}
\caption{Biological graph experiments.}
\label{kimmig}
\end{figure}
\begin{figure}
	{\includegraphics[width=.47\textwidth]{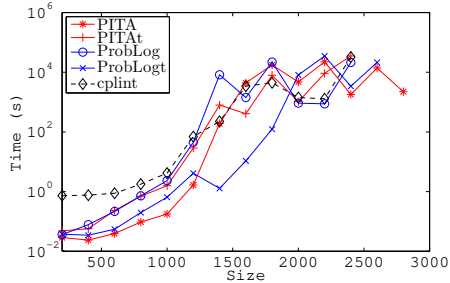}}
\caption{Average exection times on the biological graph experiments.}
\label{grapht}
\end{figure}

\begin{figure}
\centering
\subfigure
	[\texttt{bloodtype}.\label{blood}]
{\includegraphics[width=.47\textwidth]{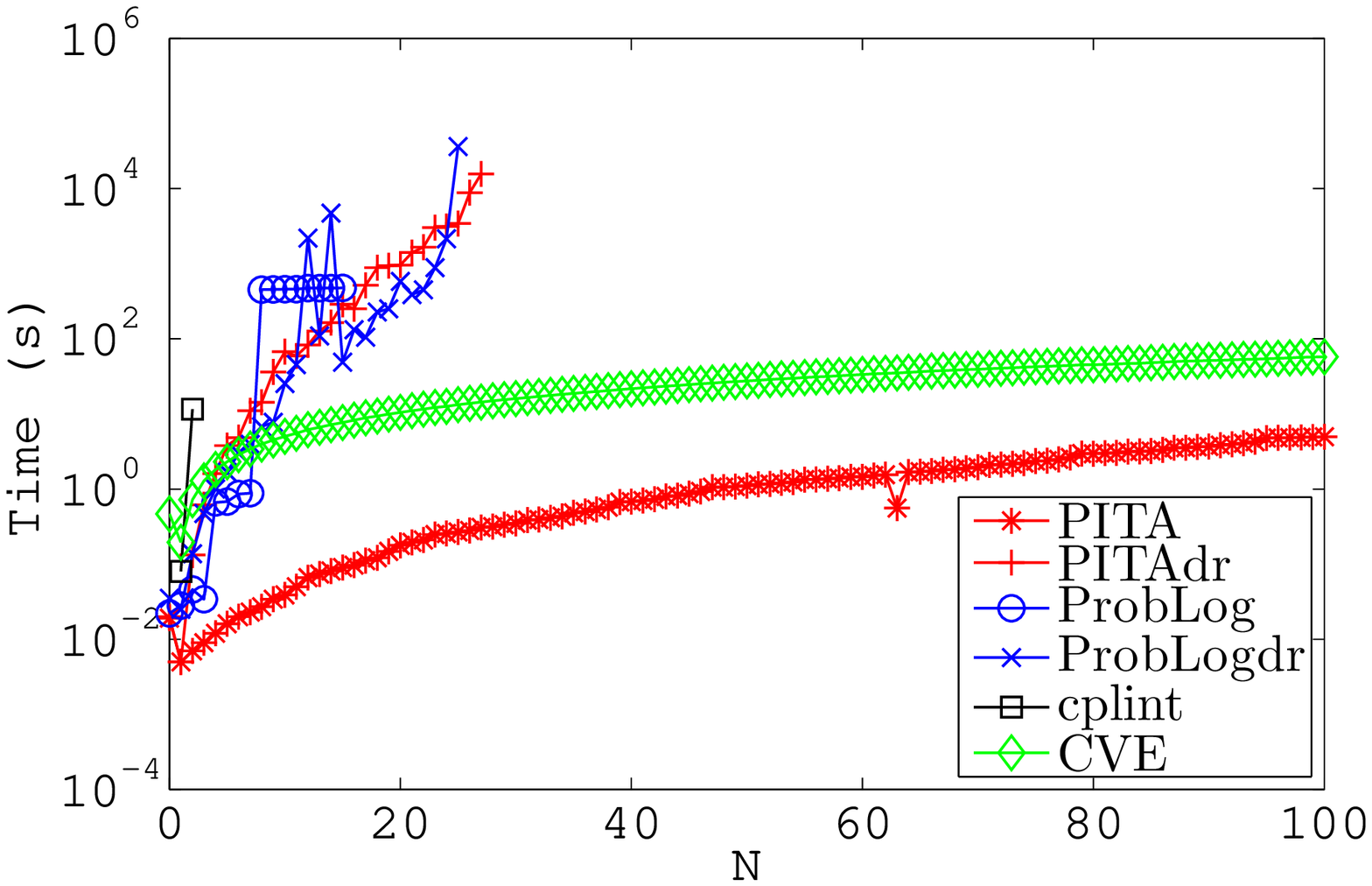}}
\hspace{0.2cm}
\subfigure
	[\texttt{growingbody}.\label{grnb}]	{\includegraphics[width=.47\textwidth]{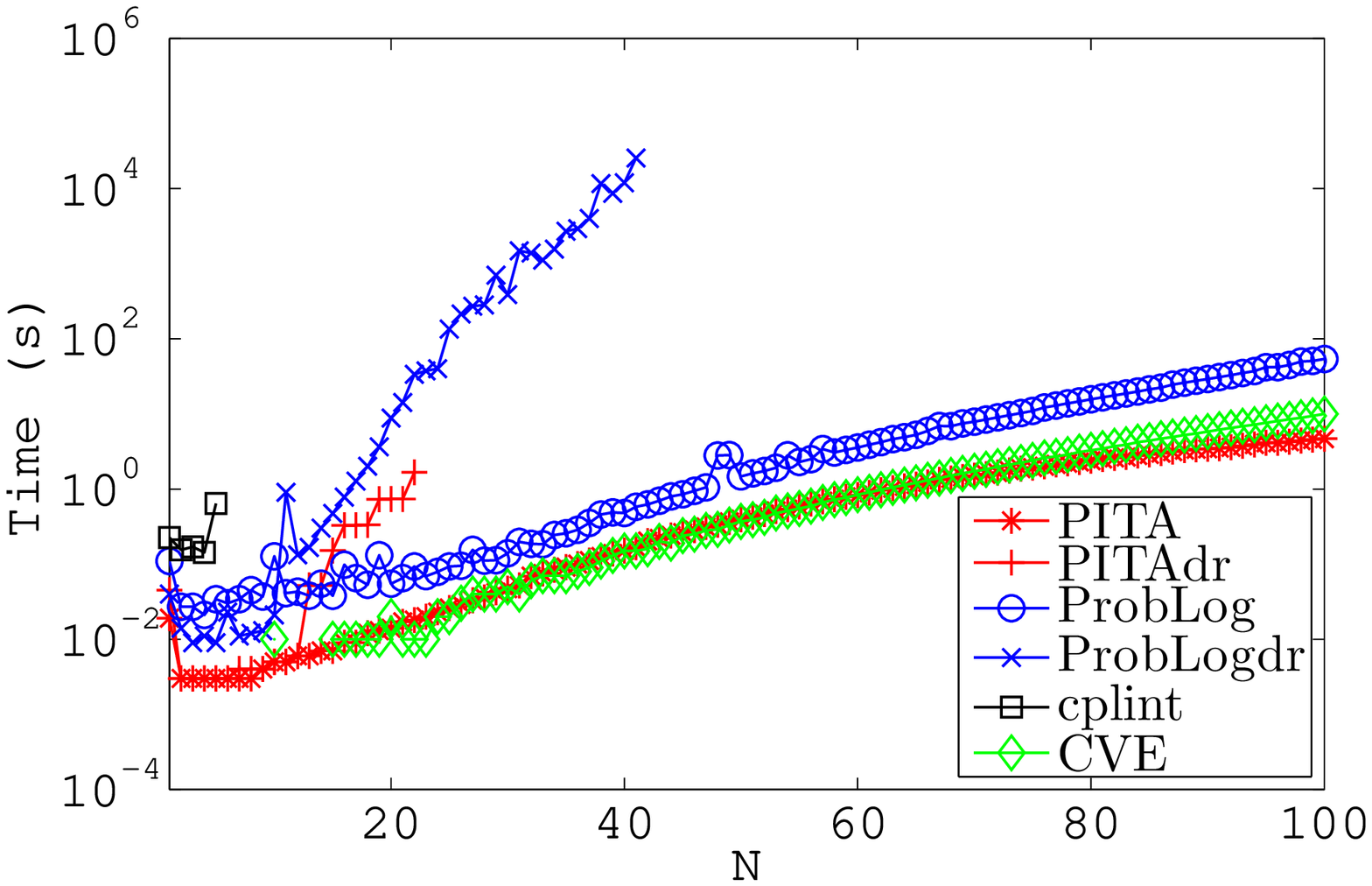}}
\caption{Datasets from (Meert et al. 2009).}
\label{wannes-1}
\end{figure}
\begin{figure}
\centering
\subfigure
	[\texttt{growinghead}.\label{grh}]
{\includegraphics[width=.47\textwidth]{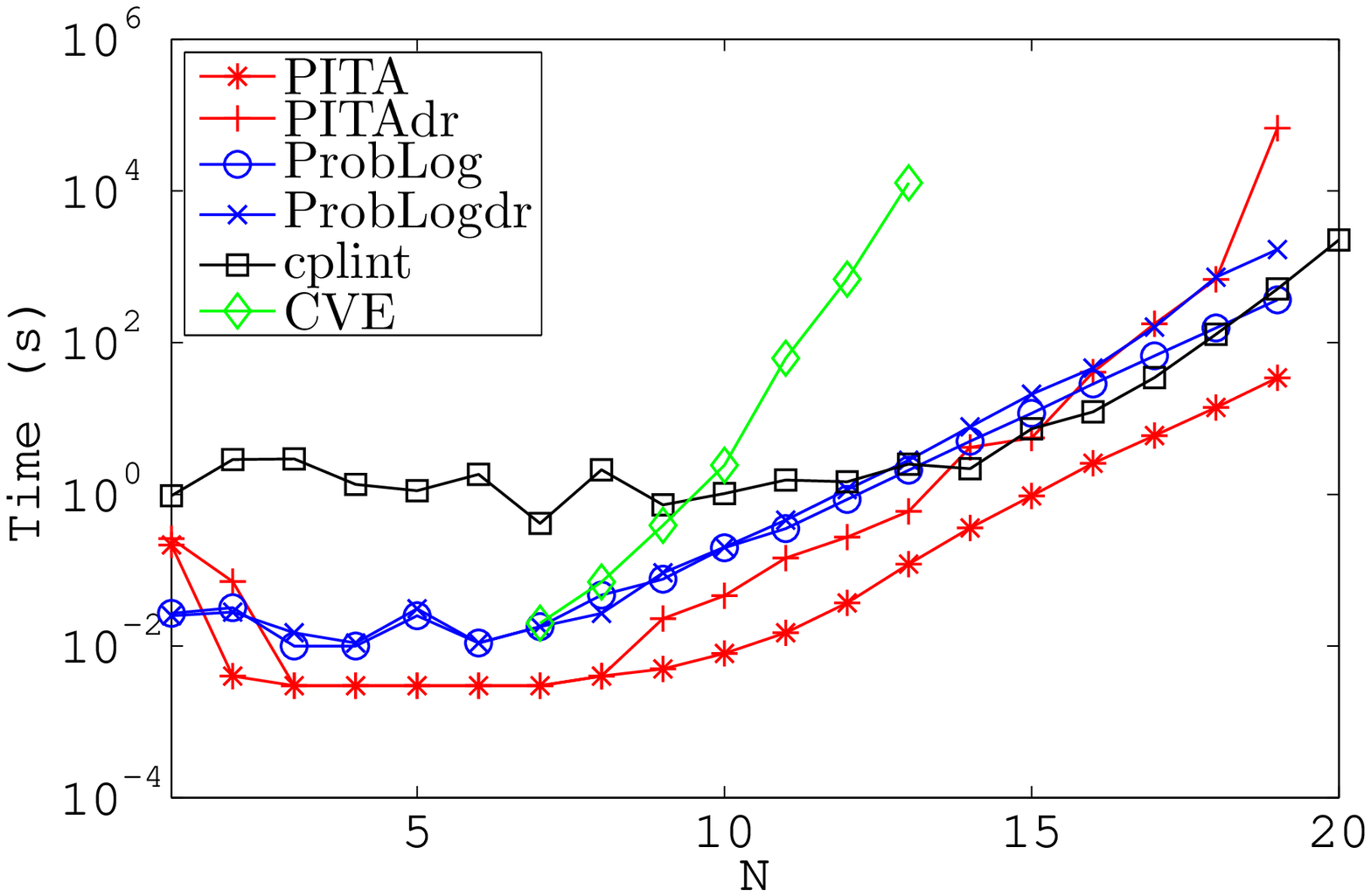}}
\hspace{0.2cm}
\subfigure
	[\texttt{uwcse}.\label{uwcse}]	{\includegraphics[width=.47\textwidth]{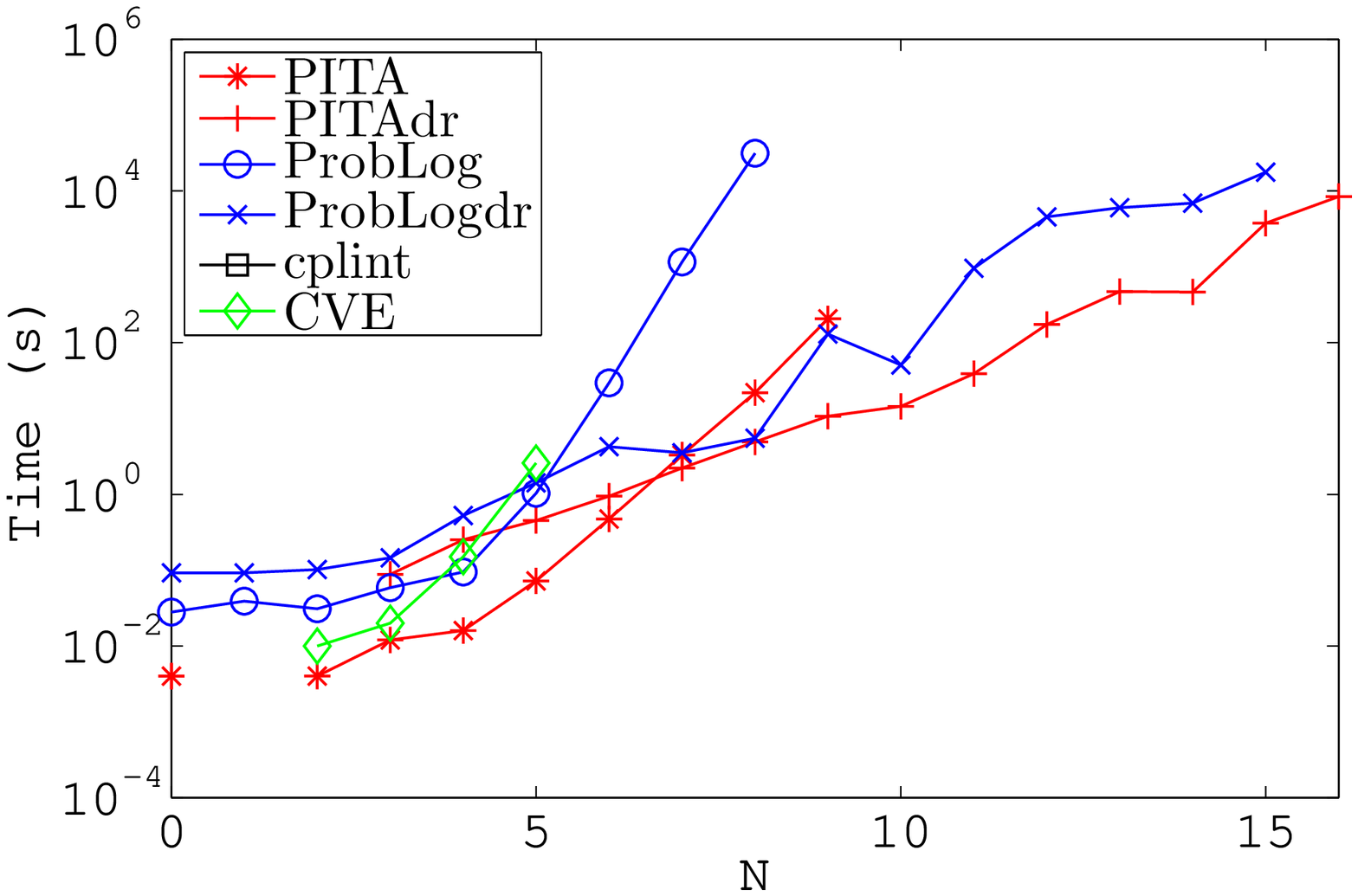}}
\caption{Datasets from (Meert et al. 2009).}
\label{wannes-2}
\end{figure}

We ran PITA, ProbLog and \texttt{cplint} on the graphs 
starting from the smallest program. In each series we stopped after
one day or at the first graph for which the program ended for lack of
memory\footnote{CVE was not applied to this dataset because the
  current version can not handle graph cycles.}.  In  \texttt{cplint}, \algorithm\ and ProbLog we used 
group sift  reordering of BDDs variables.  Figure
\ref{graph_succ_8} shows the number of subgraphs for which each
algorithm was able to answer the query as a function of the size of
the subgraphs, while Figure \ref{graph_times_8} shows the execution
time averaged over all and only the subgraphs for which all the
algorithms succeeded.  Figure \ref{grapht} alternately shows the execution times averaged, for each algorithm, over all the graphs on which the algorithm succeeded.
In these Figures PITA and PITAt refers to PITA applied to path programs (\ref{loop}) and (\ref{tab}) respectively  and similarly for ProbLog and ProbLogt.

\algorithm\ applied to program (\ref{loop}) was able to solve more subgraphs
and in a shorter time than \texttt{cplint} and all cases of ProbLog.  
On path definition (\ref{tab}), on the other hand, ProbLogt was able to solve a larger number of problems than PITAt and in a shorter time.
For PITA the
vast majority of time for larger graphs was spent on BDD maintenance.
This shows that, even if tabling consumes more memory when finding the explanations,  BDDs are built faster and use less memory,
probably due to the fact that tabling allows less redundancy (only one BDD is stored for an answer) and supports a bottom-up construction of the BDDs, which is usually better. 

The four datasets of \cite{MeeStrBlo08-ILP09-IC},
served as a final suite of benchmarks.  \texttt{bloodtype} encodes the genetic inheritance of blood type, \texttt{growingbody}  contains programs with growing bodies, \texttt{growinghead} contains programs with growing heads and \texttt{uwcse} encodes a university domain. 
The best results for ProbLog were obtained by using ProbLog's tabling in all experiments except \texttt{growinghead}.
The execution times of \texttt{cplint}, ProbLog, CVE and \algorithm\ are shown 
in Figures \ref{blood} and \ref{grnb}, \ref{grh} and \ref{uwcse}\footnote{For the missing points at the beginning of the lines a time smaller than $10^{-6}$ was recorded. For the missing points at the end of the lines the algorithm exhausted the available memory.}. In the legend PITA means that dynamic BDD variable reordering was disabled, while PITAdr has group sift automatic reordering enabled. Similarly for ProbLog and ProbLogdr.

In \texttt{bloodtype}, \texttt{growingbody} and  \texttt{growinghead} PITA without variable reordering was the fastest, while in  \texttt{uwcse} PITA with group sift automatic reordering was the fastest.
These results show that variable reordering has a strong impact on performances: if the variable order that is obtained as a consequence of the sequence of BDD operations is already good, automatic reordering severely hinders performances. Fully understanding the effect of variable reordering on performances is subject of future work.

\section{Conclusion and Future Works}
\label{conc}
This paper has made two main contributions.  The first is the
identification of \boundedtermsize{} programs and queries as
conditions for the distribution semantics to be well-defined when
LPADs contain function symbols.  As shown in
Section~\ref{sec:compare}, bounded-term-size programs and queries
sometimes include programs that other termination classes do not.
Given the transformational equivalence of LPADs and other
probabilistic logic programming formalisms that use the distribution
semantics, these results may form a basis for determining
well\--definedness beyond LPADs.

As a second contribution, the PITA transformation  provides a
practical reasoning algorithm that was directly used in the
experiments of Section~\ref{exp}.  The experiments substantiate the
PITA approach.
  Accordingly, \algorithm\  should be easily
portable to other tabling engines such as that of YAP, Ciao and B
Prolog if they support answer subsumption over general semi-lattices.
\algorithm{} is available in XSB Version 3.3 and later, downloadable from \url{http://xsb.sourceforge.net}.
A user manual is included in XSB manual and can also be found at \url{http://sites.unife.it/ml/pita}.

In the future, we plan to extend \algorithm\ to the whole class of
sound LPADs by implementing the SLG {\sc delaying} and {\sc
  simplification} operations for answer subsumption; an implementation
of tabling with term-depth abstraction (Section~\ref{tabling}) is also
underway.  Finally, we are developing a version of \algorithm\ that is
able to answer queries in an approximate way, similarly to
\cite{ProbLog-impl}.

\vspace{-0.3cm}
\bibliographystyle{acmtrans}
\bibliography{strings,bib,lpadlocal}

\appendix
%

\section{Proof of Well-Definedness Theorems (Section 4.1)}

To prove Theorem 1 we start with a lemma that states one half of the
equivalence, and also describes an implication of
the \boundedtermsize{} property for computation.

\begin{lemma}  \label{lem:finite-normal}
Let $P$ be a normal program with the \boundedtermsize{} property.
Then
\begin{enumerate}
\item Any atom in $WFM(P)$ has a finite stratum, and was computed by a
  finite number of applications of $True^P$.
\item There are a finite number of true atoms in $WFM(P)$.
\end{enumerate}
\end{lemma}
\begin{proof}
For 2), note that bounding the size of $\theta$ as used in
Definition~2 bounds the size of the ground clause $B
\leftarrow L_1,...,L_n$, and so bounds the size of $True^P_I(Tr)$ for
any $I,Tr \subseteq \cH_P$.  Since the true atoms in $WFM(P)$ are
defined as a fixed-point of $True^P_I$ for a given $I$, there must be
a finite number of them.

Similarly, since the size of $\theta$ is bounded by an integer $L$,
and since $True^P_I$ is monotonic for any $I$ $True^P_I(\emptyset)$
reaches its fixed point in a finite number of applications, and in
fact only a finite number of applications of $True^P_I$ are required to
compute true atoms in $WFM(P)$.  In addition, it can be the case that
$\cT^P_I \not = I$ only a finite number of times, so that $WFM(P)$
can contain only a finite number of strata.
\end{proof}

\noindent
{\em Theorem~1} \\
Let $P$ be a normal program.  Then $WFM(P)$ has a finite number of
true atoms iff $P$ has the \boundedtermsize{} property. 

\begin{proof}
The $\Leftarrow$ implication was shown by the previous Lemma, so that
it remains to prove that if $WFM(P)$ has a finite number of true atoms,
then $P$ has the \boundedtermsize{} property.
To show this, since the number of true atoms in $WFM(P)$ is finite,
all derivations of true atoms using $True^P_I(Tr)$ of
Definition~2 can be constructed using only a finite
set of ground clauses.  For this to be possible, the maximum term size
of any literal in any such clause is finitely bounded, so that $P$ has
the \boundedtermsize{} property.
\end{proof}

\noindent
{\em Theorem~2} \\ 
Let $T$ be a sound \boundedtermsize{} LPAD, and let $A \in
\mathcal{H}_T$.  Then $A$ has a finite set of finite explanations that
is covering.
\begin{proof}
Let $T$ be an LPAD and $w$ be a world of $T$.  Each clause
$C_{ground}$ in $w$ is associated with a choice $(C,\theta,i)$, for
which $C$ and $i$ can both be taken as finite integers.  We term
$(C,\theta,i)$ the generators of $C_{ground}$.  
%
By Theorem 1 each world $w$ of T has a finite number of true atoms,
and a maximum size $L_w$ of any atom in such a world.  We prove that
the maximum $L_{T}$ of all such worlds
has a finite upper bound.

We first consider the case in which $T$ does not contain negation.  Consider a world
$w$ whose well-founded model has the finite bound $L_{w}$ on the size
of the largest atoms. We show that $L_{w}$ can not be arbitrarily large.

Since $L_{w}$ is finite, all facts in $T$ must be ground and all
clauses range-restricted: otherwise some possible world of $T$ would
contain an infinite number of true atoms and so would not be
\boundedtermsize{} by Theorem~1. 
There must be some set $G$ of generators
which acts on a chain of interpretations $I_0 \subset I_1 \subset I_n
\subset WFM(w)$, where $I_0$ is some superset of the facts in $w$, and
the maximum size of any atom in $I_i$ is strictly increasing.  Because
$WFM(w)$ is finite and $T$ is definite, the set of generators $G$ must be
finite.

We first show that $G$ must contain generators $(C',\theta,i)$ and
$(C',\theta',j)$ for at least one disjunctive clause $C'$.  If not,
then either 1) $L_{w}$ would be infinite as there would be some
recursion in which term size increases indefinitely; or 2) if there is
no such recursion that indefinitely increases the size of terms and no
disjunctive clauses, $L_w$ could not be arbitrarily large and this
would prove the property.  In fact, without disjunction the set of
clauses causing the recursion would produce an infinite model. With
disjunction, eventually a different head is chosen and the recursion
is stopped.

Consider then, for some set $D$ of disjunctive clauses, the set
$D_{expand}$ of generators must be used to derive (perhaps
indirectly) atoms whose size is strictly greater than the maximal size
of an atom in $I_n$, while another set of generators $D_{stop}$ must
be used to stop the production of larger atoms, since $WFM(w)$ is
finite.  However, if such a situation were the case, there must also
be a world $w_{inf}$ in which for ground clauses for $D$ whose
grounding substitution is over a certain size, only the set
$D_{expand}$ of generators is chosen and $D_{stop}$ is never chosen.
The well-founded model for $w_{inf}$ would then be infinite, against
the hypothesis that $T$ is \boundedtermsize{}.

The preceding argument has shown that since there is an overall bound
on the size of the largest atom in any world for $T$, $T$ has a finite
number of different models, each of which is finite.  As each model is
finite, there is a finite number of ground clauses that determine each
model by deriving the positive atoms in the model.  Each such clause
is associated with an atomic choice, and the set of these clauses
corresponds to a finite composite choice.  The set of these composite
choices corresponding to models in which the query $A$ is true
represent a finite set of finite explanations that is covering for
$A$.

Although the preceding paragraph assumed that $T$ did not contain
negation, the assumption was made only for simplicity, so that details
of strata need not be considered.  The argument for normal programs is
essentially the same, constituting an induction where the above
argument is made for each stratum.  Because
Definition~\ref{def:lrdyn-ops} specifies that an atom can be added
to an interpretation only once, there can only be a finite number of
strata in which some true atom is added, so that there will be only a
finite number of strata overall.  Since there are only a finite number
of strata, each of which has a finite number of applications of
$True^P_I(Tr)$, a finite bound $L$ can be constructed so that $T$
fulfills the definition of \boundedtermsize .
\end{proof}

%




%

\section{Proof of the Termination Theorem for Tabling (Section~6)} 

\noindent
{\em Theorem~3} \\ Let $P$ be fixed-order
dynamically stratified normal program, and $Q$ a \boundedtermsize{}
query to $P$.  Then there is an SLG evaluation of $Q$ to $P$ using
term-depth abstraction that finitely terminates.
\begin{proof}
SLG has been proven to terminate for other notions
of \boundedtermsize{} queries, so here we only sketch the termination
proof.

First, we note that \cite{SaSW99} guarantees that if $P$ is a
fixed-order stratified program, then there is an an SLG evaluation
$\cE$ of $P$ that does not require the use of the SLG {\sc Delaying},
{\sc Simplification} or {\sc Answer Completion} operations, and by
implication no forest of $\cE$ contains a conditional answer.  Such an
evaluation is termed {\em delay-minimal}.  Note that
Definition~\ref{def:bts} constrains only the bindings used in
$True^P_I$, and these constraints may not apply ground atoms that are
undefined in the $WFM(P)$.  As a result, condition answers, if they
are not simplified or removed by {\sc Simplification} or {\sc Answer
Completion} may not have a bounded term-size.  This situation is
avoided by delay-minimal evaluations.
Next, we assume that all negative selected literals are ground.  This
assumption causes no loss of generality as the evaluation will
flounder and so terminate finitely if a non-ground negative literal is
selected.  Given this context, the proof uses the forest-of-trees
model~\cite{Swif99b} of SLG~\cite{DBLP:journals/jacm/ChenW96}.
\begin{itemize} 
\item We
consider as an induction basis the case when $Q$ is in stratum 0 --
that is, when $Q$ can be derived without clauses that contain negative
literals, or is part of an unfounded set $\cS$ of atoms and clauses
for atoms in $\cS$ do not contain negative literals.  As argued in
Section~6, the use of term-depth abstraction ensures that an SLG
evaluation $\cE$ of a query $Q$ to a program with \boundedtermsize{}
has only a finite number of trees.  In addition, since SLG works on
the original clauses of a program $P$ and $P$ is finite, (although
$ground(P)$ may not be), there can be only a finite number of clauses
resolvable against the root of any tree via {\sc Program Clause
Resolution}, and so the root of each SLG tree can contain only a
finite number of children.  Finally, to show that each interior node
has a finite number of children, we consider that there can only be a
finite number of answers to any subgoal upon which $Q$ depends.  This
follows from the fact that $\cE$ is delay-minimal and so produces no
conditional answers, together with the the bound of Definition~4 that
ensures a program is \boundedtermsize .  As a result, there are only a
finite number of nodes that are produced through {\sc Answer Return}.
These observations together ensure that each tree in any SLG forest of
$\cE$ is finite.  Since each operation (including the SLG {\sc
Completion} operation, which does not add nodes to a forest) is
applicable only one time to a given node or set of nodes in an
evaluation (i.e. executing an SLG operation removes the conditions for
its applicability) the evaluation $\cE$ itself must be finite and
statement holds for the induction basis.

\item For the induction step, we assume the statement holds for
queries whose (fixed-order) dynamic strata is less than $N$ to show
that the statement will hold for a query $Q$ at stratum $N$ as well.
As indicated above, we use a delay-minimal SLG evaluation $\cE$ that
does not require {\sc Delaying}, {\sc Simplification} or {\sc Answer
Completion} operations.  For the induction case, the various SLG
operations that do not include negation will only produce a finite
number of trees and a finite number of nodes in each tree as described
in the induction basis.  However if there is a node $N$ in a forest
with a selected negative literal $\neg A$, the SLG operation {\sc
Negation Return} is applicable. In this case, a single child will be
produced for $N$ and no further operations will be applicable to $N$.
Thus any forest in $\cE$ will have a finite number of finite trees,
and since all operations can be applied once to each node, as before
$\cE$ will be finite, so that the statement holds by induction.
\end{itemize}
\end{proof}

\section{Proof of the Correctness Theorems for PITA (Section~8)} 

The next theorem addresses the correctness of the PITA evaluation.
As discussed in Section~8, the BDDs of the PITA
transformation are represented as ground terms, while BDD operations,
such as {\em and/3}, {\em or/3}\ etc. are infinite relations on such
terms.  The PITA transformation also uses the predicate {\em
get\_var\_n/4} whose definition in Section~7 is:

\[\begin{array}{ll}
get\_var\_n(R,S,Probs,Var)\leftarrow\\
\ \ \   (var(R,S,Var)\rightarrow true;\\
\ \ \    length(Probs,L), add\_var(L,Probs,Var), assert(var(R,S,Var))).
\end{array}\]

\noindent
This definition uses a non-logical update of the program, and so
without modifications, it is not suitable for our proofs below.
Alternately, we assume that $ground(T)$ is augmented with a
(potentially infinite) number of facts of the form
$var(R,[],Var)$ for each ground rule $R$ (note that no
variable instantiation is needed in the second argument of $var/3$ if
it is indexed on ground rule names).
Clearly, the augmentation of $T$ by such facts has the same meaning as
{\em get\_var\_n/4}, but is simply done by an a priori program extension
rather than during the computation as in the implementation.

\noindent
{\em Lemma} 1\\
Let $T$ be an LPAD and $Q$ a \boundedtermsize{} query to $T$.  Then
the query $PITA_H(Q)$ to $PITA(T)$ has \boundedtermsize{}.  
\begin{proof}
Although $T_Q$ (Definition~6)
has \boundedtermsize{}, we also need to ensure that $PITA(T_Q)$
has \boundedtermsize{}, given the addition of the BDD relations {\em
and/3}, {\em or/3}, etc. along with the {\em var/3} relations
mentioned above. 

Both {\em var/3} and the BDD relations are functional on their input
arguments (i.e. the first two arguments of {\em var/3}, {\em and/3},
{\em or/3}. etc.  (cf. Section~7).  Therefore, for the
body of a clause $C$ that was true in an application of $True_I^{T_Q}$
there are exactly $n$ bodies that are true in an application of
$True_I^{PITA(T_Q)}$, where $n$ is the number of heads of $C$. Thus
the size of every ground substitutions in every iteration of
$True_I^{PITA(T_Q)}$ is bounded as well.

Note that since $PITA(T)$ and $PITA_H(Q)$ are both syntactic
transformations, the theorem applies even if the LPAD isn't sound.
\end{proof}

\comment{

By the preceding discussion, the extra literals introduced into the
bodies of rules by the $PITA$ transformation are all defined by facts.
Accordingly, for e.g. $and/3$, $and/3$ atoms that are in the $and/3$
relation are true at all points in the well-founded computation after
$True_{\emptyset}(\emptyset)$, and all $and/3$ atoms that are not in
the $and/3$ relation are false at all points after
$False_{\emptyset}(\emptyset)$.  Accordingly, derivation of an atom
$PITA_H(a)$ in $PITA(T)$ will take a finite number of steps
(i.e. applications of the $True$, $False$ and $WFM$ operators) if the
derivation of $A$ takes a finite number of steps in each possible
world of $T$, so that $PITA(T)$ is \boundedtermsize{}.
}

\noindent
{\em Theorem}~4\\
Let $T$ be a fixed-order dynamically stratified LPAD and $Q$ a ground
\boundedtermsize{} atomic query. Then there is an SLG evaluation $\cE$ of
$PITA_H(Q)$ against $PITA(T_Q)$, such that answer subsumption is
declared on $PITA_H(Q)$ using BDD-disjunction where $\cE$ finitely
terminates with an answer $Ans$ for $PITA_H(Q)$ and $BDD(Ans)$
represents a covering set of explanations for $Q$.
\begin{proof} (Sketch)
The proof uses the forest-of-trees model~\cite{Swif99b} of
SLG~\cite{DBLP:journals/jacm/ChenW96}.

Because $T$ is fixed-order dynamically stratified, queries to $T$ can
be evaluated using SLG without the {\sc delaying}, {\sc
simplification} or {\sc answer completion} operations.  Instead,
as \cite{SaSW99} shows, only the SLG operations {\sc new subgoal},
{\sc program clause resolution}, {\sc answer return} and {\sc negative
return} are needed.  Since $T$ is fixed-order dynamically stratified,
it is immediate from inspecting the transformations of
Section~\ref{algorithm} together with the fact that the BDD relations
are functional that $PITA(T)$ is also fixed-order dynamically
stratified as is $PITA(T)_Q$.

However, Theorem~3 must be extended to evaluations that include answer
subsumption, which we capture with a new operation {\sc Answer Join}
to perform answer subsumption over an upper semi-lattice $L$.  Without
loss of generality we assume that a given predicate of arity $m > 0$
has had answer subsumption declared on its $m^{th}$ argument and we
term the first $m-1$ arguments {\em non-subsuming arguments}.  We
recall that a node $N$ is an answer in an SLG tree $T$ if $N$ has no
unresolved goals and is a leaf in $T$.  Accordingly, creating a child
of $N$ with a special marker $fail$ is a method to effectively delete
an answer (cf. \cite{Swif99b}).

\begin{itemize}
\item {\sc Answer Join}: Let an SLG forest $\cF_n$
contain an answer node
\[
N = Ans \leftarrow 
\]
where the predicate for $Ans$ has been declared to use answer
subsumption over a lattice $L$ for which the join operation is
decidable, and let the arity of $Ans$ be $m>0$.  Further, let $\cA$ be
the set of all answers in $\cF_n$ that are in the same tree, $T_N$, as
$N$ and for which the non-subsuming arguments are the same as $Ans$.
Let $Join$ be the $L$-join of all the final arguments of all answers
in $\cA$.
\begin{itemize}
\item If ($Ans \leftarrow) \{arg(m,Ans)/Join\}$ is not an answer
in $T_N$, add it as a child of $N$, and add the child $fail$ to all
other answers in $\cA$.  
\item Otherwise, if
($Ans \leftarrow) \{arg(m,Ans)/Join\}$ is answer in $T_N$, create a
child $fail$ for $N$.
\end{itemize}
\end{itemize}
%
For the proof, the first item to note is that since $T_Q$
is \boundedtermsize , any clauses on which $Q$ depends that give rise
to true atoms in the well-founded model of any world of $T$ must be be
range-restricted -- otherwise since $T$ has function symbols, $T_Q$
would have an infinite model and not be \boundedtermsize.  Given this,
it is then straightforward to show that $PITA(T)_Q$ is also
range-restricted and that any answer $A$ of $PITA_H(Q)$ will be ground
(cf.  \cite{DBLP:journals/etai/Muggleton00}).  Accordingly, the
operation {\sc Answer Join} will be applicable to any subgoal with a
non-empty set of answers.

We extend Theorem~3 and Lemma~1 to show that since $PITA(T)_Q$ has
the \boundedtermsize{} property, a SLG evaluation of a query
$PITA_H(Q)$ to $PITA(T)_Q$ will terminate.  Because the join operation
for $L$ is decidable, computation of the join will not affect
termination properties.  Let $T_N$ be a tree whose root subgoal is a
predicate that uses answer subsumption.  Then each time a new answer
node $N$ is added to $T_N$ there will be one new {\sc Answer Join}
operation that becomes applicable for $N$.  Let $\cA$ be a set of
answers in $T_N$ as in the definition of {\sc Answer Join}.  Then
applying the {\sc Answer Join} operation will either 1) create a child
of $N$ that is a new answer and ``delete'' $|\cA|$ answers by creating
children for them of the form $fail$; or 2) ``delete'' the answer $N$
by creating a child $fail$ of $N$.  Clearly any answer can be deleted
at most once, and each application of the {\sc Answer Join} operation
will delete at least one answer in $T_N$.  Accordingly, if $T_N$
contains $Num$ answers, there can be at most $Num$ applications of
{\sc Answer Join} for answers in $T_N$. Using these considerations it
is straightforward to show that termination of \boundedtermsize{}
programs holds for SLG evaluations extended with answer
subsumption~\footnote{As an aside, note that due to the fact that {\sc
Answer Join} deletes all answers in $\cA$ except the join, it can be
shown by induction that immediately after an {\sc Answer Join}
operation is applied to $Ans$ in a tree $T_N$, there will be only one
``non-deleted'' answer in $T_N$ with the same non-subsuming bindings
as $Ans$.  Accordingly, if the cost of computing the join is constant,
the total cost of $Num$ {\sc Answer Join} operations will be $Num$.
Based on this observation, the implementation of PITA can be thought
of as applying an {\sc Answer Join} operation immediately after a new
answer is derived in order to avoid returning answers that are not
optimal given the current state of the computation.}.

Thus, the \boundedtermsize{} property of $PITA(T)_Q$ together with
Theorem~2 imply that there will be a finite set of finite explanations
for $PITA_H(Q)$, and the preceding argument shows that SLG extended
with {\sc Answer Join} will terminate on the query $PITA_H(Q)$.
It remains to show that an answer $Ans$ for $PITA_H(Q)$ in the final
state of $\cE$ is such that $BDD(Ans)$ represents a covering set of
explanations for $Q$.
That $BDD(Ans)$ contains a covering set of explanations can be shown
by induction on the number of BDD operations.  
For the induction basis it is easy to see that the operations {\em
zero/1} and {\em one/1} are covering for false and true atoms respectively.
\begin{itemize}
\item Consider an ``and'' operation in the body of a clause. For the
inductive assumption, $BB_{i-1}$ and $B_{i}$ both represent finite
  set of explanations covering for $L_1,\ldots,L_{i-1}$ and
 $L_i$ respectively.  Let $F_{i-1}$, $F'_i$, and $F_{i}$ be the
 formulas expressed by $BB_{i-1}$, $B_i$, and $BB_{i}$ respectively.
 These formulas can be represented in disjunctive normal form, in
 which every disjunct represents an explanation. $F_{i}$ is obtained
 by multiplying $F_{i-1}$ and $F'_{i}$ so, by algebraic manipulation,
 we can obtain a formula in disjunctive normal form in which every
 disjunct is the conjunction of two disjuncts, one from $F_{i-1}$ and
 one from $F'_i$. Every disjunct is thus an explanation for the body
 prefix up to and including $L_i$.
%
Moreover, every disjunct for $F_{i}$ is obtained by conjoining a
disjunct for $F_{i-1}$ with a disjunct for $F'_i$.
\item In the case of a ``not'' operation in the body of a clause, let
$L_i$ be the negative literal $\neg D$.  Then for $BN_{i}$  the BDD
produced by $D$, $not(BN_i,B_i)$ simply negates this BDD to produce a
covering set of explanations for $\neg D$.
\item In the case of an ``or'' operation between two answers, the resulting
BDD will represents the union of the set of explanations represented
by the BDDs that are joined.  
\end{itemize}
Since the property holds both for the induction basis and the
induction step, the set of explanations represented by $BDD(Ans)$ is
covering for the query.
\end{proof}

\comment{
there is some
ordinal $n$ such that $PITA(A)\theta$ is true in $\cE_n$ for some
substitution $\eta$, iff
\[
  BDD(PITA(A)\eta) = \bigvee \{BDD(PITA(A)\theta) |
  PITA(A)\theta\mbox{ is true in } WFM(PITA(ground(T)))\}
\]
 }

\comment{
\paragraph{Stratification induction}
Suppose $A$ is a ground atom built on a predicate of $Q_0$, and the
theorem holds for all atoms lower in the stratification ordering and
for all structured formulas build from atoms lower in the
stratification ordering. Suppose $R(g)$ is the set of all ground
instances of rules in $gr(V_0)$ with $A$ in the head.  Let
$C_r=h_0:\alpha_0\vee\ldots h_n:\alpha_n\leftarrow b_1,\ldots,b_n$ be
an element of $R(g)$ in which $h_i=g$. Now consider the clause
$t(C_r,i)$: the negative literals are relative to atoms that are lower
in the stratification ordering with respect to $A$, so they can be
completely evaluated before evaluating $A$. The positive literals have
a level that is lower or equal than $A$, however, if there is a
positive loop SLG would detect it and fail the literal. Therefore, if
there is no failure, SLG is able to find the BDDs for all the
literals. By the inductive hypothesis, the conjunction of these BDDs
represents a finite set of finite explanations that are covering for
the body. The conjunction is then conjoined with an equality that
ensures that head $i$ is selected for ground clause $r$, thus yielding
a finite set of finite explanations that are covering for $A$ if
$R(g)$ contains only $C_r$. If $R(G)$ contains other clauses, the BDDs
obtained for each clause are then disjoined to get a finite set of
explanations that are covering for $A$ }

\comment{
\paragraph{Component induction}
Suppose $A$ is a ground atom  built on a predicate of $Q_i$, and the
theorem holds for the atoms lower in the stratification ordering, for the predicates belonging to previous components and for the structured
formulas build from atoms lower in the stratification ordering and belonging to previous components.
Suppose $R(g)$ is
the set of all ground instances of rules in $gr(V_0)$ with $A$      in the head. 
Let $C_r=h_0:\alpha_0\vee\ldots h_n:\alpha_n\leftarrow b_1,\ldots,b_n$
be an element of $R(g)$ in which $h_i=g$. Suppose the literals of
previous components appear in the body before those of the current
component. By the inductive hypothesis, the BDDs associated to these
literals represent finite set of explanations that are covering for
the literals. Therefore, by means of a reasoning similar to the one
for stratification induction, the theorem holds.
}

\comment{ save for next theorem

If there are more matching facts, answer subsumption computes the
disjunction of all the involved BDDs, yielding a finite set of finite
explanations that are covering.
}

\end{document}